\documentclass[smallcondensed]{svjour3}

\smartqed  % flush right qed marks, e.g. at end of proof

\usepackage[colorlinks]{hyperref}
\usepackage{url}
\usepackage{rotating}% for sideways figures and tables
\usepackage{longtable}% for long tables
\usepackage{booktabs}
\usepackage[load-configurations=version-1]{siunitx} % newer version

\newcommand{\errorset}{{\cal E}}

\usepackage{subfigure}
\usepackage{enumerate}
\usepackage{booktabs}
\input{header}

\begin{document}

%% MLJ Title information
\title{Unconfused Ultraconservative Multiclass Algorithms}
\subtitle{}

 \author{Ugo Louche \and Liva Ralaivola}
\institute{U. Louche\at
              Qarma, Lab. d'Informatique Fondamentale de Marseille,
              CNRS, Universit{\'e} d'Aix-Marseille\\
              \email{ugo.louche@lif.univ-mrs.fr}    
           \and
           L.  Ralaivola\at
              Qarma, Lab. d'Informatique Fondamentale de Marseille,
              CNRS, Universit{\'e} d'Aix-Marseille\\
              \email{liva.ralaivola@lif.univ-mrs.fr}           %  \\
}
%%%%

\maketitle

\begin{abstract}
We tackle the problem of learning linear
classifiers from noisy datasets in a multiclass setting. The two-class version
of this problem was
studied a few years ago where the proposed
approaches to combat the noise
revolve around a Perceptron learning scheme fed with peculiar
examples computed through a weighted average of points from the
noisy training set.
We propose to build upon these approaches and we introduce a new
algorithm called \uma (for Unconfused Multiclass additive Algorithm)
which may be seen as a generalization to the multiclass
setting of the previous approaches. 
In order to characterize the noise we use the
{\em confusion matrix} as a 
multiclass extension of the classification noise studied in the
aforementioned literature.
Theoretically well-founded, \uma furthermore displays very good
empirical noise robustness, as evidenced by numerical simulations
conducted on both synthetic  and real data.
\end{abstract}

%%% Local Variables: 
%%% mode: latex
%%% TeX-master: "unconfused"
%%% End: 


%!TEX root=unconfused.tex

\section{Introduction} 
\label{sec:intro}

\paragraph{Context.} 
This paper deals with linear multiclass classification problems defined on an
input space $\inputspace$ ({\em e.g.}, $\inputspace=\realset^d$) and a set
of classes $$\classes\doteq\{1,\ldots, Q\}.$$ In particular, we are
interested in  establishing the robustness of {\em
ultraconservative additive} algorithms \cite{crammer03ultraconservative} to label noise classification in
the multiclass setting---in order to lighten notation, we will now refer to these algorithms as {\em ultraconservative algorithms}.
We study whether it  is possible to learn a linear predictor from a training set
made of independent realizations of a pair $(X,Y)$ of random variables:
$$\trainingset\doteq\{(\bfx_i,y_i)\}_{i=1}^n$$ where $y_i\in\classes$ is a corrupted version of a {\em true label}, {\em i.e.}
deterministically computed class, $t(\bfx_i)\in\classes$ associated with $\bfx_i$,
according to some {\em concept} $t$. The random noise process $Y$ that corrupts the label to provide the $y_i$'s given the $\V{x}_i$'s is
supposed uniform within each pair of classes, thus it is fully described by a
{\em confusion matrix} $\confusion=(\confusion_{pq})_{p,q}\in\realset^{Q\times Q}$ so that
$$\forall \V{x},\confusion_{pt(\V{x})}=\proba_{Y}(Y=p|\V{x}).$$
 The goal that we would like to achieve is to
provide a learning procedure able to deal with the {\em confusion
  noise} present in the  training set
$\trainingset$ to give rise  to a classifier $h$ with small risk
$$\risk(h)\doteq\proba_{X\sim\UKDistri}(h(X)\neq t(X)),$$ 
$\UKDistri$ being the distribution according to which the
$\bfx_i$'s are obtained. As we want to recover
from the confusion noise, {\em i.e.}, we want to achieve low risk on uncorrupted/non-noisy data, 
we use the term {\em unconfused} to
characterize the procedures we propose.

Ultraconservative learning procedures are online learning algorithms that
output linear classifiers. They display nice
theoretical properties regarding their convergence in the case of
linearly separable datasets, provided a sufficient separation {\em
  margin} is guaranteed (as formalized in Assumption~\ref{ass:margin}
below). In turn, these convergence-related properties
yield generalization guarantees about the quality of the predictor learned. We build upon these nice convergence
properties to show that ultraconservative algorithms are robust to a confusion noise process,
provided that: i) $\confusion$ is invertible and can be accessed, ii) the original
dataset $\lbrace (\bfx_i, t(\bfx_i) ) \rbrace_{i=1}^n$ is linearly separable.
This paper is essentially devoted to proving how/why ultraconservative
multiclass algorithms are indeed robust to such situations.
To some extent, the results provided in the present contribution may be viewed as a generalization of the contributions on
learning binary perceptrons under misclassification noise \cite{blum96polynomialtime,bylander94learning}.

Beside the theoretical questions raised by the learning setting
considered, we may depict the following example of an actual learning
scenario where learning from noisy data is relevant. This learning
scenario will be further investigated from an empirical standpoint in
the section devoted to numerical simulations (Section~\ref{sec:Expe}).
\begin{example}
\label{ex:redwire}
One  situation where coping with mislabelled data is required
arises in (partially supervised) scenarios where labelling data is very expensive. Imagine a
task of text categorization from a training set
$\trainingset=\trainingset_{\ell}\union\trainingset_u$, where 
$\trainingset_{\ell}=\{(\bfx_i,y_i)\}_{i=1}^n$ is a set of $n$
labelled training examples and $\trainingset_u=\{\bfx_{n+i}\}_{i=1}^m$ is
a set of $m$ unlabelled vectors; in order to fall back to realistic
training scenarios where more labelled data cannot be acquired, we may assume
that $n\ll m$.
A possible three-stage strategy to learn a predictor is as follows: first learn a predictor $f_{\ell}$ on
$\trainingset_{\ell}$ and estimate its confusion  error $\confusion$ {\em via} a
cross-validation procedure---$f$ is assumed to make mistakes evenly over the class regions---, second, use the learned predictor to label all the data
in $\trainingset_u$ to produce
the labelled traning set $\widehat{\trainingset}=\{(\bfx_{n+i},t_{n+i}:=f(\bfx_{n+i}))\}_{i=1}^m$ and finally,
learn a classifier $f$ from $\widehat{\trainingset}$ {\em and} the
confusion information $\confusion$.
\end{example}
This introductory example pertains to semi-supervised learning and this is
only one possible learning scenario where the contribution we propose, \uma, might
be of some use. Still, it is essential to understand right away that 
one key feature of \uma, which sets it apart from many contributions encountered in 
the realm of semi-supervised learning, is that we do provide theoretical bounds on the 
sample complexity and running time required by our algorithm to output an effective predictor.

The present paper is an extended version of~\cite{LoucheR13ACML}. Compared with the original
paper, it provides a more detailed introduction of the tools used in
the paper, a more thorough discussion on related work as well as more
extensive numerical results (which confirm the relevance of our findings). A
strategy to make use of kernels for nonlinear classification has also been added.

\paragraph{Contributions.}
Our main contribution is to show that it is both 
practically and theoretically possible to learn a
multiclass classifier on noisy data if some 
information on the noise process is available.
We propose a way to generate new points for which the true class is known.
Hence we can iteratively populate a new \emph{unconfused} dataset to learn from.
This allows us to handle a massive amount of mislabelled data with only
a very slight loss of accuracy. 
We embed our method into ultraconservative algorithms and provide a thorough
analysis of it, in which we show that the strong
theoretical guarantees that characterize
the family of ultraconservative algorithms carry over to the noisy
scenario.

\paragraph{Related Work.} Learning from mislabelled data in an iterative
manner has a long-standing history in the machine learning community.
The first contributions on this topic, based on the Perceptron
algorithm \cite{minsky69perceptrons}, are those of
\cite{bylander94learning,blum96polynomialtime,cohen97learning}, which
promoted the idea utilized here that a sample average may be used to
construct update vectors relevant to a Perceptron learning
procedure. These first contributions were focused on the binary
classification case and, for
\cite{blum96polynomialtime,cohen97learning}, tackled the specific
problem of strong-polynomiality of the learning procedure in the {\em
  probably approximately correct} (PAC) framework
\cite{kearns94introduction}. Later, \cite{stempfel07learning} proposed
a binary learning procedure making it possible to learn a kernel
Perceptron in a noisy setting; an interesting feature of this work is
the recourse to random projections in order to lower the capacity of
the class of kernel-based classifiers. Meanwhile, many advances were
witnessed in the realm of online multiclass learning procedures. In
particular, \cite{crammer03ultraconservative} proposed families of
learning procedures subsuming the Perceptron algorithm, dedicated to
tackle multiclass prediction problems. A sibling family of algorithms,
the passive-aggressive online learning algorithms~\cite{crammer06online}, 
inspired both by the previous family and the idea of minimizing
instantaneous losses, were designed to tackle various problems, among
which multiclass  linear classification. Sometimes, learning with
partially labelled data might be viewed as a problem of learning with
corrupted data (if, for example, all the unlabelled data are randomly
or arbitrarily labelled) and it makes sense to mention the
works~\cite{Kakade2008Banditron} and~\cite{RalaivolaFGBD11} as 
distant relatives to the present work.

\paragraph{Organization of the paper.} 
Section \ref{sec:setting} formally states the setting
we consider throughout this paper. Section~\ref{sec:uma}
provides the details of our main contribution: the \uma
algorithm and its detailed theoretical analysis. Section~\ref{sec:Expe}
presents numerical simulations that support the soundness of our approach.

%%% Local Variables: 
%%% mode: latex
%%% TeX-master: "unconfused"
%%% End: 

%!TEX root=unconfused.tex

%Model notation and normalization goes here

\section{Setting and Problem} \label{sec:setting}
\subsection{Noisy Labels with Underlying Linear Concept}
The probabilistic setting we consider hinges on the existence of two
components. On the one hand, we assume an
unknown (but fixed) probability distribution $\UKDistri$ on
the {\em input space} $\inputspace\doteq\realset^d$. On the other hand, we
also assume the existence of a deterministic labelling function
$t:\inputspace\to\classes$, where $\classes\doteq\{1,\ldots Q\}$, which
associates a label $t(\bfx)$ to any input example $\bfx$; in the
{\em Probably Approximately Correct} (PAC) literature, $t$ is sometimes  referred
to as a {\em concept}
\cite{kearns94introduction,valiant84theory}. 

In the present paper, we focus on learning {\em linear classifiers}, 
defined as follows.
\begin{definition}[Linear classifiers]
\label{def:linear_classifiers}
The {\em linear classifier} $f_W:\inputspace\rightarrow\classes$ is a
classifier that is associated with a set of vectors
$W=[\V{w}_1\cdots\V{w}_Q]\in\realset^{d\times Q}$, which predicts the label
$f_W(\V{x})$ of any vector $\V{x}\in\inputspace$ as
\begin{equation}
\label{eq:linearprediction}
f_W(\V{x})=\argmax_{q\in\classes}\left\langle\V{w}_q,\V{x}\right\rangle.
\end{equation}
\end{definition}

Additionally,  without loss of generality, we suppose that
 $$\proba_{X\sim\UKDistri}\left(\left\|X\right\|=1\right)=1,$$
where $\|\cdot\|$ is the Euclidean norm. This allows us to introduce the notion of margin.
\begin{definition}[Margin of a linear classifier]
\label{def:margin}
Let $c:\inputset\rightarrow\classes$ be some fixed concept. Let
$W=[\V{w}_1\cdots\V{w}_Q]\in\realset^{d\times Q}$ be a set of $Q$
weight vectors. Linear classifier $f_W$ is said to have margin $\theta>0$ with respect
to $c$ (and distribution $\UKDistri$) if the following holds:
$$\proba_{X\sim\UKDistri}\left\{\exists p\neq
  c(X):\left\langle\V{w}_{c(X)}-\V{w}_{p},X\right\rangle\leq\theta\right\}=0.$$
Note that if $f_W$ has margin $\theta>0$ with respect
to $c$  then $$\proba_{X\sim\UKDistri}(f_W(X)\neq c(X))=0.$$
\end{definition}
Equipped with this definition, we shall consider that the following
assumption of linear separability with margin $\margin$ of concept $t$ holds throughout.
\begin{assumption}[Linear Separability of $t$ with Margin $\margin$.]
\label{ass:margin}
\label{ass:linear}
There exist $\margin\geq 0$ and  $W^*=[\bfw_1^*\cdots\bfw_Q^*]\in\realset^{d\times
  Q}$, with $\Vert W^* \Vert_F^2 = 1$ ($\Vert\cdot\Vert_F$ denotes the Frobenius 
  norm) such  that $f_{W^*}$ has margin $\theta$
  with respect to the concept $t$.
\end{assumption}

In a conventional setting, one would be asked to learn a
classifier $f$ from a training set
$$\inputset_{\text{true}}\doteq\{(\V{x}_i,t(\V{x}_i))\}_{i=1}^n$$ made of $n$
labelled pairs from $\inputspace\times\classes$ such that the
$\V{x}_i$'s are independent realizations of a random variable $X$
distributed according to $\UKDistri$, with the objective of minimizing
the \emph{true risk} or {\em misclassification
  error} $\risk_{\terror}(f)$ of $f$ given by
\begin{align} 
  \risk_{\terror}(f) \doteq \proba_{X\sim\UKDistri}(f(X) \neq t(X)).
\end{align}
In other words, the objective is for $f$ to have a prediction behavior
as close as possible to that of $t$.
As announced in the introduction, there is however a little twist in the problem
that we are going to tackle. Instead of having direct access to
$\inputset_{\text{true}}$, we assume that we only have access to a
corrupted version 
\begin{align}
  \inputset\doteq\lbrace (\V{x}_i, y_i) \rbrace_{i=1}^n
\end{align}
where each $y_i$ is the realization of a random variable $Y$ whose distribution 
agrees with the following assumption:
\begin{assumption} \label{ass:law}
The law $\UKDistri_{Y|X}$ of $Y$ is the same for all $x \in \inputspace$ and its
conditional distribution $$\proba_{Y\sim\UKDistri_{Y|X=\V{x}}}(Y|X=\bfx)$$
%Ugly
is fully summarized into a {\em known} confusion matrix $\confusion$ given by
\begin{align}
  \forall \V{x},\; \confusion_{pt(\V{x})} \doteq \proba_{Y\sim\UKDistri_{Y|X=\V{x}}}(Y = p |X =
  \bfx)=\proba_{Y\sim\UKDistri_{Y|X=\V{x}}}(Y = p |t(\V{x}) = q).
\label{eq:confusion}
\end{align}
\end{assumption}

Alternatively put, the noise process that corrupts the data is {\em uniform}
within each class and its level does not depend on the precise location of
$\bfx$ within the region that corresponds to class $t(\V{x})$. 
The noise process $Y$ is both a) aggressive, as it does not
only apply, as we may expect, to regions close to the class boundaries
between classes and b) regular,  as the 
mislabelling rate is piecewise constant. Nonetheless, this setting can account
for many real-world problems as numerous noisy phenomena
can be summarized by a simple confusion matrix. Moreover it has been
proved \cite{blum96polynomialtime} that robustness to classification noise
generalizes robustness to monotonic noise where, for each class, the noise
rate is a monotonically decreasing function of the distance to the class boundaries.

\begin{remark}
The confusion matrix $\confusion$ should not be mistaken with the matrix
$\tilde{\confusion}$ of general term: $\tilde{\confusion}_{ij} \doteq
\proba_{X\sim\UKDistri_{X|Y=j}}(t(X) = i | Y = j)$ which is the class-conditional distribution of $t(X)$ given $Y$. The problem of learning from
a noisy training set and $\tilde{\confusion}$ is a different problem than the
one we aim to solve. In particular, $\tilde{\confusion}$ can be used to define
cost-sensitive losses rather directly whereas doing so with $\confusion$ is far less obvious. Anyhow, this
second problem of learning from $\tilde{\confusion}$ is far from trivial and very interesting, and it falls way beyond the scope of the present work.
\end{remark}

Finally, we assume the following from here on:
\begin{assumption} \label{ass:1} 
$\confusion$ is invertible.
\end{assumption}
Note that this assumption is not as restrictive as it may appear. For
instance, if we consider the learning setting depicted in
Example~\ref{ex:redwire} and implemented in the numerical simulations, 
then the confusion matrix obtained from the
first predictor $f_{\ell}$ is often diagonally dominant, {\em i.e.} the
magnitudes of the diagonal entries are larger than the sum of the
magnitudes of the entries in their corresponding rows, and $\confusion$ is therefore
invertible. Generally speaking, the problems that we are interested in ({\em i.e.}
problems where the true classes seems to be recoverable) tend to have
invertible confusion matrix. It is most likely that invertibility is merely a sufficient
condition on $\confusion$ that allows us to establish learnability in the
sequel. Identifying less stringent conditions on $\confusion$, or conditions termed in a different way---which would for instance be based on the condition number of $\confusion$---for learnability to remain,  is a research issue of its own that
we leave for future investigations.

The setting we have just presented allows us to view $\inputset=\{(\bfx_i,y_i)\}_{i=1}^n$ as the
realization of a random sample $\{(X_i,Y_i)\}_{i=1}^n$, where each
pair $(X_i,Y_i)$ is an independent copy of the random pair $(X,Y)$ of
law $\UKDistri_{XY}\doteq\UKDistri_X\UKDistri_{X|Y}.$

\subsection{Problem: Learning a Linear Classifier from Noisy Data}
The problem we address is the learning of a classifier $f$ from
$\inputset$ {\em and} $\confusion$ so that the error
rate $$\risk_{\terror}(f)=\proba_{X\sim\UKDistri}(f(X)\neq t(X))$$ of
$f$ is as small as possible: the usual goal of learning a
classifier $f$ with small risk is preserved, while now the training data
is only made of corrupted labelled pairs.

Building on Assumption~\ref{ass:linear}, we may refine our learning
objective by restricting ourselves to linear classifiers $f_W$, for
$W=[\V{w}_1\cdots\V{w}_Q]\in\realset^{d\times Q}$ (see
Definition~\ref{def:linear_classifiers}). %with
%\begin{equation}
%\label{eq:linearprediction}
%\forall\bfx\in\inputspace,\;
%f_W(\bfx)=\argmax\nolimits_{q\in\classes}\langle\bfw_q,\bfx\rangle.
%\end{equation}
Our goal is thus to learn a relevant matrix $W$ from
$\inputset$ {\em and} the confusion matrix $\confusion$. More precisely, we
achieve risk minimization through classic additive methods and the core of
this work is focused on computing noise-free update points such that the
properties of said methods are unchanged.

%%% Local Variables: 
%%% mode: latex
%%% TeX-master: "unconfused"
%%% End: 
 %aka definition
%!TEX root=unconfused.tex

\section{{\sc Uma}: Unconfused Ultraconservative Multiclass Algorithm}
\label{sec:uma}
This section presents the main result of the paper, that is, the \uma
procedure, which is a response to the problem posed above: \uma makes
it possible to learn a multiclass linear predictor
from $\inputset$ and the confusion information $\confusion$. In
addition to the algorithm itself, this section provides theoretical
results regarding the convergence and sample complexity of \uma.

As \uma is a generalization of the ultraconservative additive online
algorithms proposed in \cite{crammer03ultraconservative} to the case
of noisy labels, we first and foremost recall the essential features of this
family of algorithms. The rest of the section is then devoted to the
presentation and analysis of \uma.

\subsection{A Brief Reminder on Ultraconservative Additive Algorithms}
Ultraconservative additive online algorithms were introduced by Crammer et
al. in \cite{crammer03ultraconservative}. As already stated, these
algorithms output multiclass linear predictors $f_W$ as in Definition~\ref{def:linear_classifiers}
and their purpose is therefore to compute a set
$W=[\V{w}_1\cdots\V{w}_Q]\in\realset^{d\times Q}$ of $Q$ weight vectors
from some training sample $\trainingset_{\text{true}}=\{(\V{x}_i,t(
\V{x}_i ) )\}_{i=1}^n$.
To do so, they implement the procedure depicted in Algorithm~\ref{alg:ua}, which centrally revolves
around the identification of an {\em error set} and its simple update: when processing a training
pair $(\V{x},y)$, they perform updates of the
form
$$\V{w}_q\leftarrow\V{w}_q+\tau_q\V{x}, \; q=1,\ldots
Q,$$
whenever the {\em error set}
 $\errorset(\V{x},y)$ defined as
\begin{equation}
\label{eq:errorset}
\errorset(\V{x},y)\doteq\left\{r\in\classes\backslash\{y\}:
\langle\bfw_r,\V{x}\rangle-\langle\bfw_{y},\V{x}\rangle\geq
0\right\}
\end{equation}
is not empty, with the constraint for the family $\{\tau_q\}_{q\in \classes}$ of {\em
  step sizes} to fulfill
\begin{equation}
\left\{\begin{array}{l}\tau_{y}=1\\ \tau_r\leq 0, \text{ if }
r\in\errorset(\V{x},y)\\ \tau_r=0,\text{ otherwise}\end{array}\right.\quad
\text{ and }\quad\sum_{r=1}^Q\tau_r=0.
\label{eq:tau}
\end{equation}
The term {\em ultraconservative} refers to the fact that only those
prototype vectors $\V{w}_r$ which achieve a larger inner product $\langle\V{w}_r,\V{x}\rangle$ than $\langle\V{w}_y,\V{x}\rangle$, that is, the 
vectors that can entail a prediction mistake when decision rule~\eqref{eq:linearprediction} is applied, 
 may be affected by the update procedure. The term {\em
  additive} conveys the fact that the updates consist in modifying the
weight vectors $\V{w}_r$'s by adding a portion of $\V{x}$ to
them (which is to be opposed to multiplicative update schemes). Again,
as we only consider these additive types of updates in what follows,
it will have to be implicitly understood even when not explicitly mentioned.

\begin{algorithm}[tb]
\caption{Ultraconservative Additive algorithms
\cite{crammer03ultraconservative}.\label{alg:ua}}
\begin{algorithmic}
\REQUIRE $\trainingset_{\text{true}}$
\ENSURE $W=\left[\V{w}_1, \ldots, \V{w}_Q \right]$ and associated classifier $f_W(\cdot)=\argmax_{q}\langle\bfw_q,\cdot\rangle$\vspace{3mm}
\STATE Initialization: $\V{w}_q\leftarrow 0,\;\forall q\in\classes$
\REPEAT % {some stopping criterion is not met}
\STATE access training pair $(\V{x}_t,y_t)$
\STATE compute the error set $\errorset(\V{x}_t,y_t)$ according to~\eqref{eq:errorset}
\IF{$\errorset(\V{x}_t,y_t)\neq\emptyset$}
\STATE compute a set $\{\tau_q\}_{q\in\classes}$ of update steps
that comply with~\eqref{eq:tau}
\STATE perform the updates
$$\V{w}_q \leftarrow \V{w}_q +\tau_q \V{x}_q,\;\forall q\in\classes$$
\ENDIF
\UNTIL{some stopping criterion is met}
\end{algorithmic}
\end{algorithm}

One of the main results regarding ultraconservative algorithms, which
we extend in our learning scenario is the following.
\begin{theorem}[Mistake bound for ultraconservative algorithms \cite{crammer03ultraconservative}.]
\label{th:ua}
Suppose that concept $t$ is in accordance with
Assumption~\ref{ass:margin}. The number of mistakes/updates made by
one pass over $\trainingset$ by any ultraconservative procedure
is upper-bounded by $2/\theta^2$.
\end{theorem}
This result is essentially a generalization of the well-known
Block-Novikoff theorem~\cite{block62perceptron,novikoff63convrgence}, which establishes a mistake bound for
the Perceptron algorithm (an ultraconservative algorithm itself).
\subsection{Main Result and High Level Justification}
This section presents our main contribution, \uma, a theoretically grounded noise-tolerant
multiclass algorithm depicted in Algorithm~\ref{alg:uma}. \uma learns
and outputs a matrix $W=[\V{w}_1\cdots\V{w}_Q]\in\realset^{d\times Q}$
from a noisy training set $\trainingset$ to produce the associated linear
classifier
\begin{equation} \label{eq:additive}
f_W(\cdot)=\argmax_q\langle\bfw_q,\cdot\rangle
\end{equation}
by iteratively
updating the $\V{w}_q$'s, whilst maintaining $\sum_q\V{w}_q=0$
throughout the learning process. As a new member of multiclass additive algorithms, we
may readily recognize in step~\ref{step:uc1} through
step~\ref{step:uc2} of Algorithm~\ref{alg:uma} the 
generic step sizes $\{\tau_q\}_{q\in\classes}$ promoted by ultraconservative algorithms
(see Algorithm~\ref{alg:ua}).
 An important feature of \uma
is that it only uses information provided by $\trainingset$ and does
not make assumption on the accessibility to the noise-free dataset $\trainingset_{\text{true}}$: the incurred pivotal difference with regular ultraconservative algorithms is that the update points used are now the computed (line~\ref{step:mistakestart} through line~\ref{step:mistakeend}) $\bfz_{pq}$ vectors
instead of the $\bfx_i$'s. Establishing that under some conditions \uma stops and provides a
classifier with small risk when those update points are used is the purpose of the following
subsections; we will also discuss the unspecified
step~\ref{step:select}, dealing with the selection step.

\begin{algorithm}[tb]
\caption{\uma:  Unconfused Ultraconservative Multiclass
Algorithm.\label{alg:uma}} \begin{algorithmic}[1]
\REQUIRE $\inputset= \{(\V{x}_i, y_i)\}_{i=1}^{n}$,
confusion matrix $\confusion\in\realset^{Q\times Q}$, and $\alpha>0$
\ENSURE $W=\left[\V{w}_1, \ldots, \V{w}_K \right]$ and classifier $f_W(\cdot)=\argmax_{q}\langle\bfw_q,\cdot\rangle$\vspace{3mm}
\STATE $\V{w}_k\leftarrow 0$, $\forall k\in\classes$
\REPEAT % {some stopping criterion is not met}
\STATE select $p$ and $q$ \label{step:select}
\STATE compute set $\setA{p}^{\alpha}$ as  $$\setA{p}^{\alpha} \leftarrow 
\left\{ \V{x} |\V{x}\in\inputset ,
        \dotProd{\V{w}_p}{\V{x}} 
        - \dotProd{\V{w}_k}{\V{x}} 
        \geq \alpha,\;\forall k\neq p\right\}$$\label{step:mistakestart}
\STATE for $k=1,\ldots,Q$, compute $\gamma_{k}^p$ as
$$\gamma_k^p\leftarrow\frac{1}{n}\sum_{i=1}^{n}\indicator{y_i=k}\indicator{\V{x}_i\in
\setA{p}^{\alpha}}\V{x}_i^{\top},\;\forall k\in\classes$$\label{step:gamma_kp} 
\STATE form $\Gamma^p\in\realset^{Q\times d}$ as
$$\Gamma^p\leftarrow\left[\gamma_1^p\cdots \gamma_Q^p\right]^{\top},$$
\STATE  compute the update vector $\xup{pq}$ according to 
($[M]_q$ refers to the $q$th row of matrix $M$) 
$${\V{z}}_{pq}\leftarrow ([\confusion^{-1}\Gamma^p]_q)^{\top},$$
\label{step:mistakeend}
\STATE compute the error set $\errorset^{\alpha}(\V{z}_{pq},q)$ as\\
$$\errorset^{\alpha}(\V{z}_{pq},q)\leftarrow\{r\in\classes\backslash\{q\}:
\langle\bfw_r,\V{z}_{pq}\rangle-\langle\bfw_q,\V{z}_{pq}\rangle\geq\alpha\}$$\label{step:uc1}
\IF{$\errorset^{\alpha}(\V{z}_{pq},q)\neq \emptyset$}
\STATE compute some ultraconservative update steps
$\tau_1,\ldots,\tau_Q$ such that:
 $$\left\{\begin{array}{l}\tau_q=1\\ \tau_r\leq 0,\forall
r\in\errorset^{\alpha}(\V{z}_{pq},q)\\ \tau_r=0,\text{
  otherwise}\end{array}\right.\quad\text{and}\quad \sum_{r=1}^Q\tau_r=0$$\label{step:uc2}
\STATE perform the updates for $r=1,\ldots,Q$:\\
 $$\V{w}_r \leftarrow \V{w}_r +\tau_r \xup{pq}$$
\ENDIF
%\hspace{2mm} $\V{w}_q \leftarrow \V{w}_q + \xup{pq}$
\UNTIL{$\|{\bf z}_{pq}\|$ is too small}\label{step:criterion}
\end{algorithmic}
\end{algorithm}

For the impatient reader, we may already leak some
of the ingredients we use to prove the relevance of our procedure.
Theorem~\ref{th:ua}, which shows the convergence of ultraconservative
algorithms, rests on the analysis of the updates made when
training examples are misclassified by the current classifier. The conveyed
message is therefore that examples that are erred upon are central to
the convergence analysis. It turns out that
steps~\ref{step:mistakestart} through~\ref{step:mistakeend} of \uma
(cf. Algorithm~\ref{alg:uma}) construct a point $\xup{pq}$ that is,
with high probabilty,  mistaken on. More
precisely, the true class $t(\xup{pq})$ of $\xup{pq}$ is $q$ and it is predicted to be of
class $p$ by the current classifier; at the same time, these
update vectors are guaranteed to realize a positive margin
condition with respect to $W^*$:
$\langle\V{w}_q^*,\xup{pq}\rangle>\langle\V{w}_k^*,\xup{pq}\rangle$
for all $k\neq q$. The ultraconservative feature of the algorithm is
carried by step~\ref{step:uc1} and step~\ref{step:uc2}, which make it
possible to update any prototype vector $\bfw_r$ with $r\neq q$ having
an inner product $\langle\bfw_r,{\bf
  z}_{pq}\rangle$ with ${\bf z}_{pq}$ larger than $\langle\bfw_q,{\bf
  z}_{pq}\rangle$ (which should be the largest if a correct prediction
were made). The reason why we have results `with high
probability' is because the $z_{pq}$'s are sample-based estimates
of update vectors known to be of class $q$ but predicted as being of
class $p$, with $p\neq q$;
computing the accuracy of the sample estimates is one of the important
exercises of what follows. A control on the accuracy makes it possible
for us to then establish the convergence of the proposed algorithm.

%We now turn to the formal analysis of $\uma$.

\subsection{With High Probability, ${\bf z}_{pq}$ is a Mistake with Positive Margin} \label{sec:UpnAl}
Here, we prove that the update vector $\xup{pq}$ given in
step~\ref{step:mistakeend} is, with high  probability, a point on
which the current classifier errs.

\begin{proposition}
\label{prop:tilde_update}
Let $W=[\V{w}_1\cdots \V{w}_Q]\in\realset^{d\times Q}$ and $\alpha>
0$ be fixed.
Let $\setA{p}^{\alpha}$ be defined as in step~\ref{step:mistakestart} of Algorithm~\ref{alg:uma}, i.e:
\begin{equation}
\setA{p}^{\alpha}\doteq \left\{ \V{x} |\V{x}\in\inputset , 
        \dotProd{\V{w}_p}{\V{x}} 
        - \dotProd{\V{w}_k}{\V{x}} 
        \geq \alpha,\;\forall k\neq p\right\}.
\label{eq:Ap}
\end{equation}
For $k\in\classes$, $p\neq k$, consider the random
variable $\gamma_{k}^p$ ($\gamma_{k}^p$ in step~\ref{step:gamma_kp} of Algorithm~\ref{alg:uma} is a realization
of this variable, hence the overloading of notation $\gamma_{k}^p$):
$$\gamma_{k}^p\doteq\frac{1}{n}\sum_{i}\indicator{Y_i=k}\indicator{X_i\in\setA{p}^{\alpha}}X_i^{\top}.$$
The following holds, for all $k\in\classes$:
\begin{equation}
\expectation_{\trainingset}\left\{\gamma_{k}^p\right\}=\expectation_{\{(X_i,Y_i)\}_{i=1}^n}\left\{\gamma_{k}^p\right\}=\sum_{q=1}^Q\confusion_{kq}\mu_q^p,
\end{equation}
where
\begin{equation}
	\label{eq:muqp}\mu_q^p\doteq\expectation_{\UKDistri_X}\left\{\indicator{t(X)=q}\indicator{X\in\setA{p}^{\alpha}}X^{\top}\right\}.
\end{equation}
\end{proposition}
\begin{proof}
Let us compute
$\expectation_{\UKDistri_{XY}}\{\indicator{Y=k}\indicator{X\in\setA{p}^{\alpha}}X^{\top}\}$:
\begin{align*}
\expectation_{\UKDistri_{XY}}&\{\indicator{Y=k}\indicator{X\in\setA{p}^{\alpha}}X^{\top}\}\\
&=\int_{\inputspace}\sum_{q=1}^Q\indicator{q=k}\indicator{\V{x}\in\setA{p}^{\alpha}}\V{x}^{\top}\proba_{Y}(Y=q|X=\V{x})d\UKDistri_{X}(\V{x})\\
&=\int_{\inputspace}\indicator{\V{x}\in\setA{p}^{\alpha}}\V{x}^{\top}\proba_{Y}(Y=k|X=\V{x})d\UKDistri_{X}(\V{x})\\
&=\int_{\inputspace}\indicator{\V{x}\in\setA{p}^{\alpha}}\V{x}^{\top}\confusion_{kt(\bfx)}d\UKDistri_{X}(\V{x})\tag{cf.~\eqref{eq:confusion}}\\
&=\int_{\inputspace}\sum_{q=1}^Q\indicator{t(\V{x})=q}\indicator{\V{x}\in\setA{p}^{\alpha}}\V{x}^{\top}\confusion_{kq}d\UKDistri_{X}(\V{x})\\
&=\sum_{q=1}^Q
\confusion_{kq}\int_{\inputspace}\indicator{t(\V{x})=q}\indicator{\V{x}\in\setA{p}^{\alpha}}\V{x}^{\top}d\UKDistri_{X}(\V{x})
=\sum_{q=1}^Q \confusion_{kq}\mu_q^p,
\end{align*}
where the last line comes from the fact that the classes are
non-overlapping. The $n$ pairs $(X_i,Y_i)$ being identically
and independently distributed gives the result.\qed
\end{proof}

Intuitively, $\mu_q^p$  must be seen as an example of class $p$ which is erroneously predicted
as being of class $q$. Such an example 
is precisely what we are looking for to update the current classifier;
as expecations cannot be computed, the estimate $\xup{pq}$  of
$\mu_q^p$ is used instead of $\mu_q^p$. 

\begin{proposition} \label{prop:tilde_update2}
Let $W=[\V{w}_1\cdots \V{w}_Q]\in\realset^{d\times Q}$ and $\alpha\geq
0$ be fixed.
For $p,q\in\classes$, $p\neq q$, ${\V{z}}_{pq}\in\realset^d$ is
such that
\begin{align}
&\expectation_{\UKDistri_{XY}}{{\V{z}}_{pq}}=\mu_{q}^p\\
&\langle\V{w}_q^*,\mu_{q}^p\rangle -
\langle\V{w}_k^*,\mu_{q}^p\rangle\geq\theta,\;\forall
k\neq q,\label{eq:mumargin}\\
&\langle\V{w}_p,\mu_{q}^p\rangle -
\langle\V{w}_k,\mu_{q}^p\rangle\geq \alpha,\;\forall k\neq p.\label{eq:muerror}
\end{align}
(Normally, we should consider the transpose of $\mu_{q}^p$, but since
we deal with vectors of $\realset^d$---and not matrices---we abuse
the notation and omit
the transpose.)

This means that 
\begin{enumerate}[i)]
\item $t(\mu_{q}^p)=q$, {\em i.e.} the `true' class of $\mu_{q}^p$ is $q$;
\item  and $f_W(\mu_{q}^p)=p$; $\mu_{q}^p$ 
is therefore misclassified by the current classifier $f_W$.
\end{enumerate}
\end{proposition}
\begin{proof}
According to Proposition~\ref{prop:tilde_update},
$$\expectation_{\UKDistri_{XY}}\left\{\Gamma^{p}\right\}=\expectation_{\UKDistri_{XY}}\left\{\left[\begin{array}{c}\gamma_1^p\\
    \vdots\\ \gamma_Q^p\end{array}\right]\right\}=\left[\begin{array}{c}\expectation_{\UKDistri_{XY}}\left\{\gamma_1^p\right\}\\
    \vdots\\
    \expectation_{\UKDistri_{XY}}\left\{\gamma_Q^p\right\}\end{array}\right]=\left[\begin{array}{c}\sum_{q=1}^Q\confusion_{1q}\mu_1^p\\\vdots\\\sum_{q=1}^Q\confusion_{Qq}\mu_Q^p\end{array}\right]=\confusion
    \left[\begin{array}{c}\mu_1^p\\\vdots\\\mu_Q^p\end{array}\right].$$

Hence, inverting $\confusion$ and extracting the $q$th of the
resulting matrix equality gives that $\expectation\left\{{\V{z}}_{pq}\right\}=\mu_{q}^p$.

Equation~\eqref{eq:mumargin} is obtained thanks to
Assumption~\ref{ass:linear} combined with~\eqref{eq:muqp} and the linearity of the
expectation. Equation~\eqref{eq:muerror} is obtained thanks to the
definition~\eqref{eq:Ap} of $\setA{p}^{\alpha}$ (made of points that are predicted to be of
class $p$)  and the linearity of the expectation.\qed
\end{proof}

The attentive reader may notice that Proposition
\ref{prop:tilde_update2} or, equivalently, step~\ref{step:mistakeend}, is
precisely the reason for requiring $\confusion$ to be invertible, as the
computation of $\xup{pq}$ hinges on the resolution of a system of
equations based on $\confusion$. 

\begin{proposition}
\label{prop:zpqerror}
Let $\varepsilon>0$ and $\delta\in(0;1]$.
There exists a number $$
 n_0(\varepsilon, \delta, d, Q) = \bigO\left( \frac{1}{\varepsilon^2} \left[
\ln\frac{1}{\delta} + \ln Q + d\ln\frac{1}{\varepsilon}
\right] \right) $$
 such that if the number of training
samples is greater than $n_0$ then, with high probability
\begin{align}
&\langle\V{w}_q^*,\xup{pq}\rangle -
\langle\V{w}_k^*,\xup{pq}\rangle 
\geq \theta - \varepsilon
\label{eq:zmargin} \\
&\langle\V{w}_p,\xup{pq}\rangle -
\langle\V{w}_k,\xup{pq}\rangle\geq 0,\;\forall k\neq p.\label{eq:zerror}
\end{align}
\end{proposition}
\begin{proof} 
The existence of $n_0$ relies on 
pseudo-dimension arguments.
 We defer this
part of the proof to Appendix~\ref{apd:first} and
we will directly assume here that if $n \geq n_0$,
then, with probability $1 - \delta$ for any $\V{W}$, $\xup{pq}$.
\begin{align}
  \left| \dotProd{\V{w}_p - \V{w}_q}{\xup{pq}} -
  \dotProd{\V{w}_p - \V{w}_q}{\mupt{q}{p}} \right| \leq
 \varepsilon. \label{eq:zproof}
\end{align}
Proving \eqref{eq:zmargin} then proceeds by observing that
\begin{align*}
\dotProd{\V{w}_q^* - \V{w}_k^*}{\xup{pq}} = \dotProd{\V{w}_q^* - \V{w}_k^*}{\mupt{q}{p}}
+ \dotProd{\V{w}_q^* - \V{w}_k^*}{\xup{pq} - \mupt{q}{p}}
\end{align*}
bounding the first part using Proposition~\ref{prop:tilde_update2}:
$$\dotProd{\V{w}_q^* - \V{w}_k^*}{\mupt{q}{p}} \geq \theta
$$
and the second one with \eqref{eq:zproof}.
A similar reasoning allows us to get \eqref{eq:zerror} by setting
$\alpha \doteq \varepsilon$ in $\setA{p}^{\alpha}$ .\qed
\end{proof}
This last proposition essentially says that the update vectors ${\bf
  z}_{pq}$ that we compute are, with high probability, erred upon and
realize a margin condition $\theta - \varepsilon$.

Note that $\alpha$ is needed to cope with the imprecision incurred by the
use of empirical estimates.  Indeed, we can
only approximate $\langle\V{w}_p,\xup{pq}\rangle -
\langle\V{w}_k,\xup{pq}\rangle$ in \eqref{eq:zerror} up to a precision of
$\varepsilon$. Thus for the result to hold we need to have
$\langle\V{w}_p,\mu_{q}^p\rangle - \langle\V{w}_k,\mu_{q}^p\rangle \geq \varepsilon$
which is obtained from \eqref{eq:muerror} when $\alpha = \varepsilon$.
In practice, this just says that the points used in the computation of $\xup{pq}$
are at a distance at least $\alpha$ from any decision boundaries.

\begin{remark}
\label{rem:alpha}
It is important to understand that the parameter $\alpha$ helps us
derive sample complexity results by allowing us to retrieve a 
linearly separable training dataset with {\em positive} margin 
from the noisy dataset. The theoretical results we prove hold for any
   such $\alpha>0$ parameter and the smaller this parameter, the larger the sample complexity, {\em i.e.}, the harder it is for the algorithm to take advantage 
   of a training samples that meets the sample complexity requirements. In other 
   words, the smaller $\alpha$, the less likely it is for \uma to succeed; yet,
    as shown in the experiments, where we use $\alpha=0$, \uma continues to
    perform quite well.     
\end{remark}

\subsection{Convergence and Stopping Criterion}

We arrive at our main result, which provides both
convergence and a stopping criterion. 
\begin{proposition}
\label{prop:convergence}
Under Assumptions \ref{ass:margin}, \ref{ass:law} and \ref{ass:1} there exists a
number $n$, polynomial in $d, 1/\theta, Q, 1/\delta$, such that if the 
training sample is of size at least $n$, then, with high probability ($1 -
\delta$), \uma makes at most $\bigO(1/{\theta}^2)$ updates.
\end{proposition}
\begin{proof}
Let $\inputset_{\xup{}}$ the set of all the update vectors $\xup{pq}$
generated during the execution of \uma and labeled with their \emph{true} class $q$.
Observe that, in this context, \uma (Alg. \ref{alg:uma}) behaves like  a
regular ultraconservative algorithm run on $\inputset_{\xup{}}$. Namely: a) lines
\ref{step:mistakestart} through \ref{step:mistakeend} compute a new point in
$\inputset_{\xup{}}$, and b) lines \ref{step:uc1} through \ref{step:uc2} perform an
ultraconservative update step.

From Proposition~\ref{prop:zpqerror}, we know that with high probability, $w^*$ is a classifier
with positive margin $\theta - \varepsilon$ on $\inputset_{\xup{}}$ and it comes from
Theorem~\ref{th:ua} that \uma does not make more than
$\bigO(1/{\theta}^2)$ mistakes on such dataset.

Because, by construction, we have that with high probability each element of $\inputset_{\xup{}}$ is erred upon
then $\vert \inputset_{\xup{}} \vert \in
\bigO(1/{\theta}^2)$; that means that, with high probability, \uma does not make
more than $\bigO(1/{\theta}^2)$ updates.

All in all, after $\bigO(1/{\theta}^2)$ updates, there is a high
probability that we are not able to construct examples on which \uma makes a mistake or,
equivalently, the conditional misclassification errors $\proba(f_{W}(X)=p|Y=q)$
are all small. \qed
\end{proof}

Even though \uma operates in a batch setting, it `internally' simulates the
execution of an online algorithm that encounters a new training point
(\changeok{$\inputset_{\xup{}}$}{$\xup{pq} \in \inputset_{\xup{}}$}) at each time
step. To more precisely see how \uma can be seen as an online algorithm, it
suffices to imagine it be run in a way where each vector update is made after a
chunk of $n$ (where $n$ is as in Proposition~\ref{prop:convergence}) training
data has been encountered and used to compute \changeok{}{the next element of}
$\inputset_{\xup{}}$.
Repeating this process  $\bigO(1/{\theta}^2)$ times then guarantees convergence
with high probability. Note that, in this scenario, \uma requires $n' =
\bigO(n/{\theta}^2)$ data to converge which might be far more than the sample
complexity exhibited in Proposition~\ref{prop:convergence}. Nonetheless, $n'$
still remains polynomial in $d$, $1/{\theta}$, $Q$ and $1/{\delta}$. For more
detail on this (online to batch conversion) approach, we refer the interested
readers to \cite{blum96polynomialtime}.

\subsection{Selecting $p$ and $q$} \label{sec:pqselect}

So far, the question of 
selecting good pairs of values $p$ and $q$ to perform updates has been
left unanswered. Indeed, our results hold for \emph{any} pair $(p,q)$ and
convergence is guaranteed even when $p$ and $q$ are arbitrarily selected as long
as $\xup{pq}$ is not $\mathbf{0}$. Nonetheless, it is reasonable to use
heuristics for selecting $p$ and $q$ with the hope that it might 
improve the practical convergence speed.

 On the one hand, we may focus on the pairs $(p,q)$ for which the 
empirical misclassification rate
\begin{equation}
\hat{\proba}_{X\sim\trainingset}\left\{f_W(X) \neq
t(X)\right\}\doteq\frac{1}{n}\sum_{i=1}^n\indicator{f_W(\V{x}_i)\neq
t(\V{x}_j)}
\label{eq:empirical_misclassification}
\end{equation}
is the highest ($X\sim\trainingset$
  means that $X$ is randomly drawn from the uniform distribution of
  law $\V{x}\mapsto n^{-1}\sum_{i=1}^n\indicator{\V{x}=\V{x}_i}$
  defined with respect to training set $\trainingset=\{(\V{x}_i,y_i)\}_{i=1}^n$). We want to favor those pairs $(p,q)$
because, i) the induced update may lead to a greater reduction of the error and
ii) more importantly, because $\xup{pq}$ may be more reliable, as
$\setA{p}^{\alpha}$ will be bigger.

On the other hand, recent advances in the passive aggressive
literature \cite{Ralaivola12}  have
emphasized the importance of minimizing the empirical confusion rate,
given for a pair $(p,q)$ by the quantity
\begin{equation}
\hat{\proba}_{X\sim\trainingset}\left\{f_W(X)=p|t(X)=q\right\}\doteq\frac{1}{n_q}\sum_{i=1}^n\indicator{t(\V{x}_i)=q,
f_W(\V{x}_i)=p},
\label{eq:empirical_confusion}
\end{equation}
where
$$n_q\doteq\sum_{i=1}^n\indicator{t(\V{x}_i)=q}.$$
This approach is especially worthy when dealing
with imbalanced classes and one might want to optimize the selection of $(p,q)$
with respect to  the confusion rate. 

Obviously, since the true labels in the training data cannot be
accessed, neither of the quantities defined
in~\eqref{eq:empirical_misclassification}
and~\eqref{eq:empirical_confusion} can be computed.  Using a
result provided in~\cite{blum96polynomialtime}, which states that the
norm of an update vector computed as ${\bf z}_{pq}$ directly provides an
estimate of~\eqref{eq:empirical_misclassification}, we devise two possible strategies for selecting $(p,q)$:
\begin{align}
  (p,q)_{\terror} &\doteq \argmax_{(p,q)} \|{\bf z}_{pq}\|\\
  (p,q)_{\tconf} &\doteq \argmax_{(p,q)}\frac{\|{\bf
      z}_{pq}\|}{\hat{\pi}_q},
\end{align}
where $\hat{\pi}_q$ is the estimated proportion of examples of true
class $q$ in the training sample. In a way similar to the computation of
${\bf z}_{pq}$ in Algorithm~\ref{alg:uma}, $\hat{\pi}_q$ may be
estimated as follows:
$$\hat{\pi}_q=\frac{1}{n}[\confusion^{-1}{\hat{\bfy}}]_q,$$
where $\hat{\bfy}\in\realset^Q$ is the vector containing the number of
examples from $\inputset$ having noisy labels $1,\ldots,Q$, respectively.

The second selection criterion is intended to normalize the number of
errors with respect to the proportions of different classes and aims at
being robust to imbalanced data. Our goal here is to provide a way to take into
account the class distribution for the selection of $(p,q)$. Note that
this might be  a first step towards transforming \uma into an algorithm 
for minimizing the confusion risk, even though additional (and significant) 
work is required to provably provide \uma with this feature.

On a final note, we remark that $(p,q)_{\tconf}$ requires
additional precautions when used: when $(p,q)_{\terror}$ is implemented,
$\xup{pq}$ is guaranteed to be the update vector of maximum norm among all
possible update vectors, whereas this no longer holds true when
$(p,q)_{\tconf}$ is used and if $\xup{pq}$ is close to $\mathbf{0}$ then
there may exist another possibly more informative---from the standpoint of
convergence speed---update vector $\xup{p'q'}$ for some $(p',q')\neq(p,q).$

\subsection{\uma and Kernels}
\label{sec:umakernels}
Thus far, we have only considered the situation where linear
classifiers are learned. There are however many
learning problems that cannot be handled effectively without going
beyond linear classification. A popular strategy to deal with such a
situation is obviously to make use of kernels~\cite{schoelkopf02learning}. In
this direction, there are (at least) two paths that can be taken. The
first one is to revisit \uma and provide a kernelized algorithm based
on a dual representation of the weight vectors, as is done with the kernel Perceptron (see~\cite{cristianini00introduction}) or
its close cousins (see,
{\em e.g.} \cite{cristianini98kernel-adatron,Dekel05theforgetron,freund99large}). Doing
so would entail the question of finding sparse expansions of
the weight vectors with respect to the training data in order to contain
the prediction time and to derive generalization guarantees based on
such sparsity: this is an interesting and ambitious research program
on its own. A second strategy, which we make use of in the
numerical simulations, is simply to build upon the idea of Kernel
Projection Machines~\cite{blanchard08finite,conf/cvpr/TakerkartR11}:
first, perform a Kernel Principal Component Analysis (shorthanded as
kernel-\pca afterwards) with $D$ principal axes, second, project the
data onto the principal $D$-dimensional subspace and, finally, run
\uma on the obtained data. The availability of numerous methods to
efficiently extract the principal subspaces (or approximation thereof) \cite{bach02kernel,DrineasKM06SIAM,drineas05nystrom,stempfel07learning,williams01nystrom}
makes this path a viable strategy to render \uma  usable for
nonlinearly separable concepts. This explains why we decided to use
this strategy in the present paper.

%%% Local Variables: 
%%% mode: latex
%%% TeX-master: "unconfused"
%%% End: 
 %aka Algo
%!TEX root=unconfused.tex
\section{Experiments} \label{sec:Expe}

In this section, we present results from numerical
simulations of our approach and we
discuss different practical aspects of \uma. The ultraconservative
step sizes retained are those corresponding to a regular Perceptron:
$\tau_p=-1$ and $\tau_q=+1$, the other values of $\tau_r$ being equal
to $0$.

Section \ref{sec:expe:toy} discusses robustness results, based
on simulations conducted on synthetic data while 
Section \ref{sec:expe:reuters} takes it a step further and
evaluates our algorithm on real data, with a realistic
noise process related to Example~\ref{ex:redwire} (cf. Section~\ref{sec:intro}).

We essentially use what we call the {\em confusion rate} as a performance
measure, which is :
\[
\frac{1}{\sqrt{Q}} \fro{\H{\Conf}}
\]
Where $\fro{\H{\Conf}}$ is the Frobenius norm of the confusion
matrix $\H{\Conf}$ computed on a test set $S_{\text{test}}$  (independent from
the training set), {\em i.e.}: $$\fro{\H{\Conf}}^2 = \sum_{i,j}
\H{\Conf}_{ij}^2,\text{ with }\H{\Conf}_{pq} \doteq\left\{\begin{array}{ll}
0 &\text{ if } p=q,\\
\displaystyle\frac{\sum_{\V{x}_i \in S_{\text{test}}} \indicator{\H{y}_i = p
\text{ and } t_i
= q }} {\sum_{\V{x}_i \in S_{\text{test}}} \indicator{t_i = q}}&\text{
otherwise,}\end{array}\right.$$ with $\widehat{y}_i$ the label predicted
for the test instance ${\bf x}_i$ by the learned predictor. $\H{\Conf}$ is much
akin to a recall matrix, \changeok{}{and the $1/\sqrt{Q}$ factor ensure that
the confusion rate is comprised within $0$ and $1$}.

\subsection{Toy dataset} \label{sec:expe:toy}

We use a $10$-class dataset with a total of roughly $1,000$
$2$-dimensional examples uniformly distributed according to
$\uniform$, which is the uniform distribution over the unit circle
centered at the origin.
Labelling is achieved according to \eqref{eq:linearprediction} given a 
set of $10$ weight vectors $\V{w}_1,\ldots,\V{w}_{10}$, which are also randomly generated  according to  $\uniform$;
all these weight vectors have therefore norm $1$.
%That way, each weight vector $\V{w}_i$, for $i=1,\ldots,10$, has norm $1$, 
%which will prevent too high (and artificial) discrepancies between
%the margins between the different classes at hand. 
A margin
$\margin = 0.025$ is enforced in the generated data by removing examples
that are too close to the decision boundaries---practically, with this value of
$\margin$, the case
where three classes are so close to each other that
no training example from one of the classes remained after enforcing the margin
never occurred.

The learned classifiers are tested against a dataset
of $10,000$ points that are distributed according to the training distribution.
 The results reported in the tables and graphics are averaged over $10$ runs.

The noise is generated from the sole confusion
matrix. This
situation can be tough to handle and is rarely met with real
data but we stick with it as it is a good example of a
worst-case scenario.

\paragraph{Robustness to noise.} We first (Fig.~\ref{fig:noise_robustness})
evaluate the robustness to noise of \uma by running our algorithm with various confusion matrices. 
We uniformly draw a reference nonnegative square matrix $M$,
the rows of $M$ are then normalized, {\em i.e.} each entry of $M$ is divided by the
sum of the elements of its row, so $M$ is a stochastic matrix. If $M$ is not invertible  it is
rejected and we draw a new matrix until we have an invertible one. Then, we
define $N$ such that $N = {(M - I)}/10$, where $I$ is the identity matrix of
order $Q$; typically $N$ has nonpositive diagonal entries and nonnegative
off-diagonal coefficients. We will use $N$ to parametrize a family of confusion
matrices that have their most dominant coefficient to move from their diagonal
to their off-diagonal parts. Namely, we run \uma $20$ times with confusion matrices
$\confusion\in\{\confusion_i\doteq \Omega(I + iN)\}_{i=1}^{20}$, where $\Omega$
is a matrix operator which outputs a (row-)stochastic matrix: when applied on
matrix $A$, $\Omega$ replaces the negative elements of $A$ by zeros and it
normalizes the rows of the obtained matrix; note that $i = 10$ corresponds to
the case where $\confusion = M$. Equivalently, one can think of $\confusion_i$
as the weighted average between $I$ and $\Omega(N)$ where $I$ has a constant
weight of $1$ and $\Omega(N)$ is weighted by $i$. Note that, after some point,
further increasing $i$ has little effect on $\confusion_i$ as it eventually converges to
$\Omega(N)$. Figure \ref{fig:noise_robustness} plots our results against the
Frobenius norm of the diagonal-free confusion matrix $\confusion$, that is:
$\norm{\confusion - \texttt{diag}(\confusion)}_F$ where
$\texttt{diag}(\confusion)$ denotes the diagonal matrix with the same diagonal values
as $\confusion$. \changeok{}{For the sake of comparison, we also have run \uma with a fixed
confusion matrix $\Conf = I$ on the same data. This amounts to running a
Perceptron through the data multiple times and it allows us to have a baseline for measuring
the improvement induced by the use of the confusion matrix.
}

\paragraph{Robustness to the incorrect estimation of the confusion matrix.}  
The
second experiment (Fig.~\ref{fig:noise_estimation}) evaluates the robustness of
\uma to the use of a confusion matrix that is not exactly the confusion matrix
that describes the noise process corrupting the data; this will allow us to
measure the extent to which a confusion matrix (inaccurately) estimated from the 
training data can be dealt with by \uma. Using the
same notation as before, and the same idea of generating a random stochastic
reference matrix $M$, we proceed as follows:  we use the given matrix $M$
to corrupt the noise-free dataset and then, each confusion matrix from the
family  $\{\Conf_i\}_{i=1}^{20}$ is fed to \uma as if it were the confusion
matrix governing the noise process.
We introduce the notion of {\em approximation} factor $\rho$ as $\rho(i)\doteq
1-i/10$, so that $\rho$ takes values in the set $\{-1,-0.9,\ldots,0.9\}$. As
reference, the limit case where $\rho = 1$---that is, $i = 0$---corresponds
to the case where \uma is fed with the identity matrix $I$, effectively being
oblivious of any noise in the training set. More generally, the values of
$\confusion$ are being shifted away from the diagonal
as $\rho$ decreases, the equilibrium point being $\rho = 0$ where $\confusion$ is
equal to the \emph{true} confusion matrix $M$. Consequently, a positive (resp.
negative) approximation factor means that the noise is underestimated (resp.
overestimated), in the sense that the noise process described by $\confusion$
would corrupt a lower (resp. higher) fraction of labels from each class than the \emph{true}
noise process applied on the training set, and corresponding to $M$.
Figure~\ref{fig:noise_estimation} plots the confusion rate against this approximation factor.

\begin{figure}
  \centering
  \subfigure[Robustness to noise\label{fig:noise_robustness}]
{
%    \centering
    \includegraphics[width=0.48\textwidth]{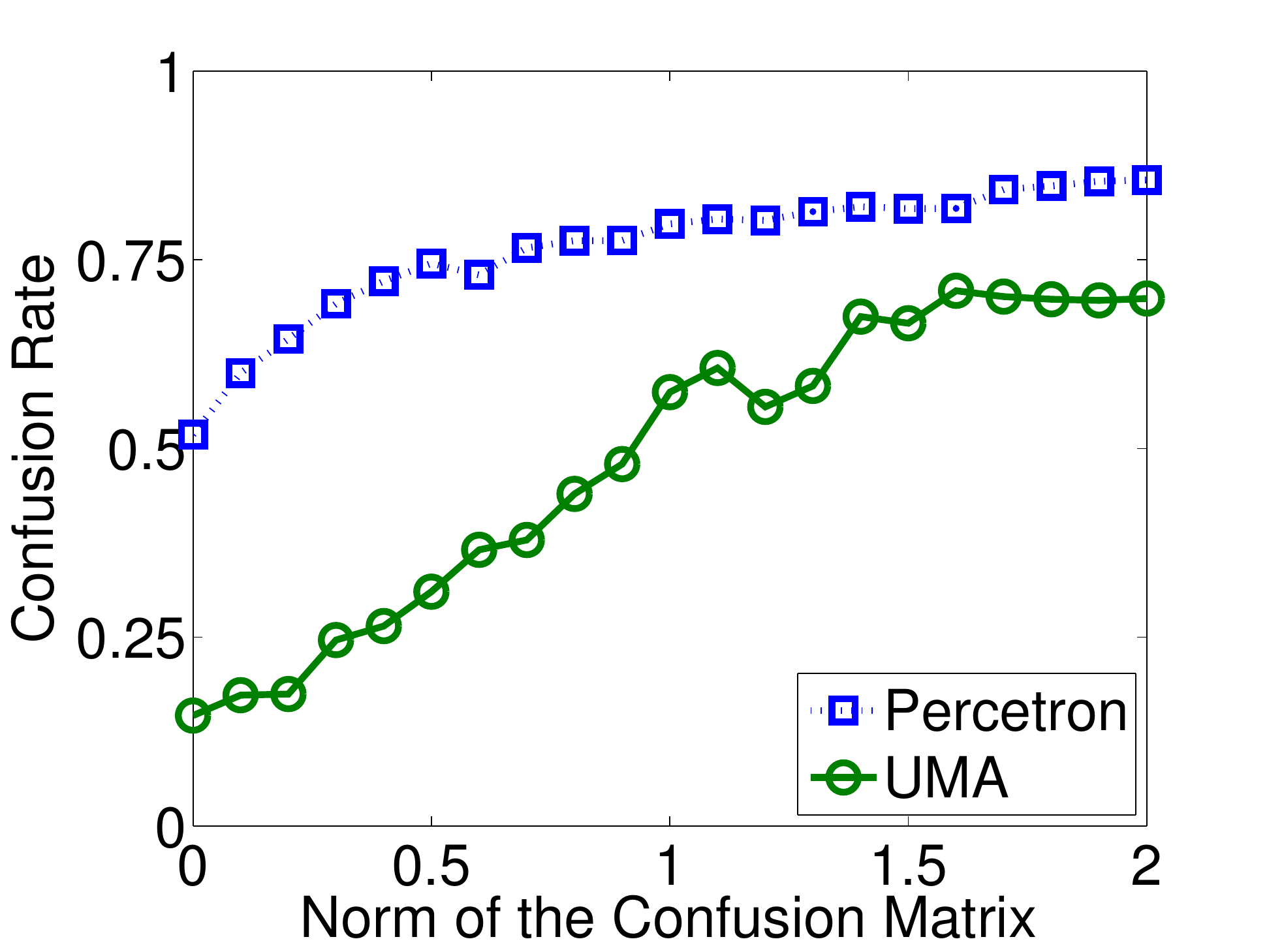}
} 
\subfigure[Robustness to noise estimation\label{fig:noise_estimation}]
{
    \includegraphics[width=0.48\textwidth]{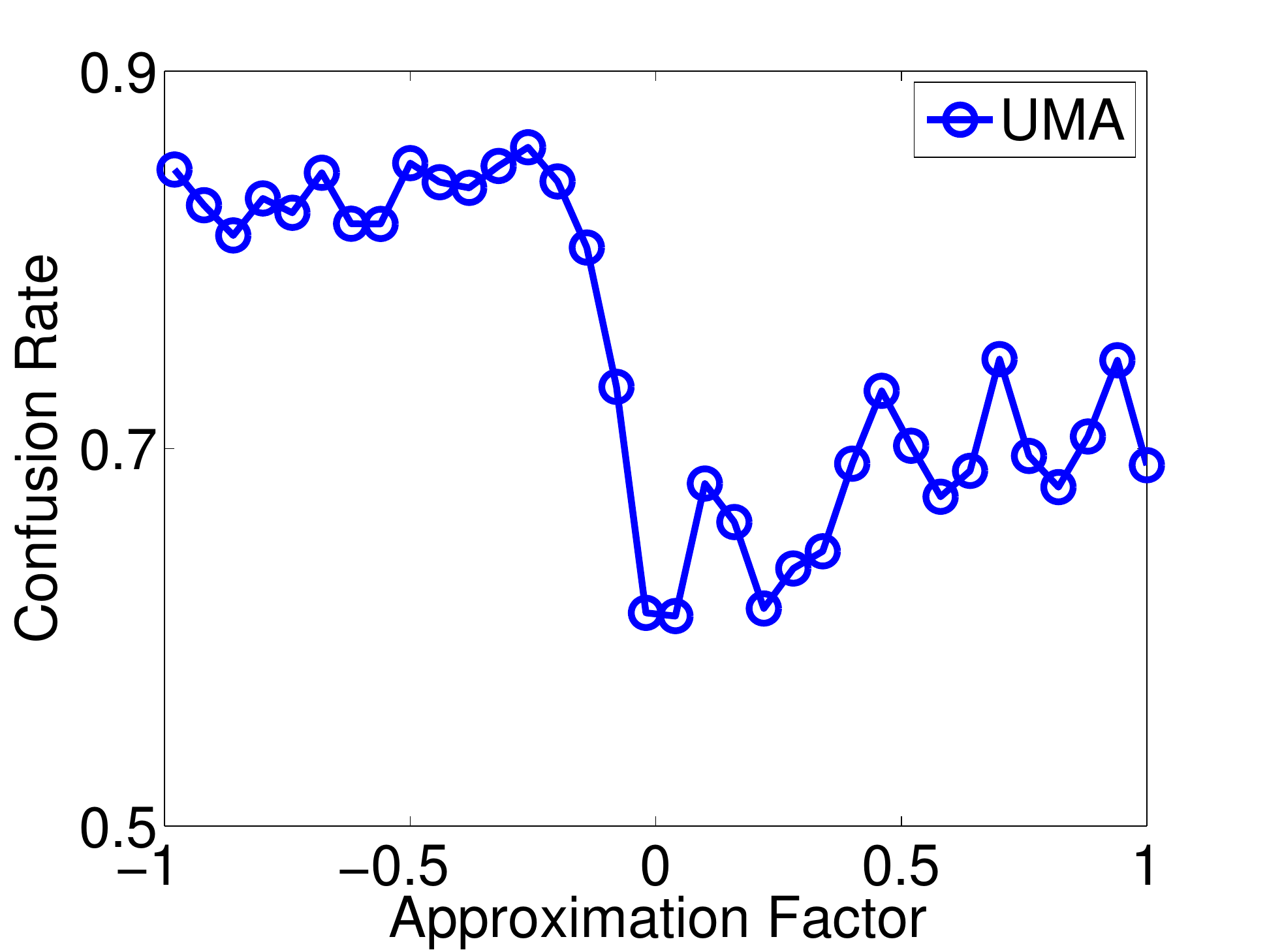}
}
  \caption{ (a) evolution of the confusion rate (y-axis)
for different noise
    levels (x-axis); (b) evolution of the same quantity with 
respect to errors in the confusion matrix $\Conf$ (x-axis) measured by the approximation factor (see text).}
\end{figure}

On Figure~\ref{fig:noise_robustness} we observe that \uma clearly provides improvement over the
Perceptron algorithm for every noise level tested, as it achieves lower confusion
rates. Nonetheless, its performance degrades as the noise level increases, going
from a confusion rate of $0.5$ for small noise levels---that is, when
$\norm{\confusion - \texttt{diag}(\confusion)}_F$ is small---to roughly
$2.25$ when the noise is the strongest.
Comparatively, the Perceptron algorithm follows the same trend, but with higher confusion rate, ranging from $1.7$ to $2.75$.

The second
simulation (Fig. \ref{fig:noise_estimation}) points out that, in addition to
being robust to the noise process itself,  \uma is also robust to underestimated
(approximation factor $\rho > 0$) noise levels, but not to overestimated (approximation factor $\rho
< 0$) noise levels.
Unsurprisingly, the best confusion rate corresponds to an approximation factor of $0$, which means that \uma is
using the true confusion matrix and can achieve a confusion rate as low as
$1.8$. There is a clear gap between positive and negative approximation
factors, the former yielding confusion rates around $2.6$ while the latter's are
slightly lower, around $2.15$. From these observations, it is clear that the
approximation factor has a major influence on the performances of \uma. 
%when the true confusion matrix is not available.

\subsection{Real data}
\label{sec:expe:reuters}

\subsubsection{Experimental Protocol} \label{sec:modus}
In addition to the results on synthetic data, we also perform simulations in a
realistic learning scenario. In this section we are going to assume that
labelling examples is very expensive and we implement the strategy
evoked in Example~\ref{ex:redwire}. More precisely, for a given dataset
$\trainingset$, proceed as follows:
\begin{enumerate}
  \item Ask for a small number $m$ of examples for each of the $Q$ classes.
  \item Learn a rough classifier\footnote{For the sake of self-containedness,
  we use \uma for this task (with $\Conf$ being the identity matrix). Remind
  that, when used this way, \uma acts as a regular Perceptron algorithm} $g$
  from these $Q \times m$ points.
  \item Estimate the confusion $\confusion$ of $g$ on a small labelled subset $\trainingset_{\text{conf}}$
  of $\trainingset$.
  \item Predict the missing labels $\bfy$ of $\trainingset$ using $g$; thus, $\bfy$ is  
  a sequence of noisy labels.
  \item Learn the final classifier $f_{\uma}$ from $\trainingset$, $\bfy$,
  $\confusion$ and measure its error rate.
\end{enumerate}
One might wonder why we do not simply sample a very small portion of
$\trainingset$ in the first step. The reason is that in the case of very
uneven classes proportions some of the classes may be missing in this
first sampling. This is problematic when estimating $\confusion$ as it leads
to a non-invertible confusion matrix.
Moreover, the purpose of $g$ is only to provide a baseline for the computation
of $\bfy$, hence tweaking the class (im)balance in this step is not a problem.

In order to put our results into perspective, we compare them with
results obtained from various
algorithms. This allows us to give a precise idea of the benefits and
 limitations of \uma. Namely, we learn four additional classifiers: $f_{\bfy}$ is a
regular Perceptron learned on $\trainingset$ labelled with noisy labels $\bfy$, $f_{\text{conf}}$
and $f_{\text{full}}$ are trained with the correctly labelled training sets $\trainingset_{\text{conf}}$ and
$\trainingset$ respectively and, lastly, $f_{\SSVM}$ is a classifier produced by
a multiclass semi-supervised \emph{SVM} algorithm (\SSVM, \cite{Bennett}) run on
$\trainingset$ where only the labels of $\trainingset_{\text{conf}}$ are provided.
The performances achieved by $f_{\bfy}$ and $f_{\text{full}}$ provide bounds for \uma's error rates: on the one hand,
$f_{\bfy}$ corresponds to a worst-case situation, as we simply ignore the confusion matrix
and use the regular Perceptron instead---arguably, \uma should perform better
than this---; on the other hand,
$f_{\text{full}}$ represents the best-case scenario for learning, when all the correct labels are
available---the performance of $f_{\text{full}}$ should always top that of \uma (and the performances of other classifiers).
The last two classifiers, $f_{\text{conf}}$ and $f_{\text{\SSVM}}$, provide us with
objective comparison measures. They are learned from the same data as \uma but use
them differently: $f_{\text{conf}}$ is learned from the reduced training set
$\trainingset_{\tconf}$ and $f_{\text{\SSVM}}$ is output by a semi-supervised learning strategy that infers both $f_{\text{\SSVM}}$ and the missing labels of $\trainingset$ and it totally ignores the predictions $\bfy$ made by $g$.
Note
that according to the learning scenario we implement, we assume $\confusion$ to be
estimated from raw data. This might not always be the case with real-world problems and
$\confusion$ might be easier and/or less expensive to get than raw data; for instance, 
it might be deduced from expert knowledge on the studied domain. In that case, $f_{\text{conf}}$ and $f_{\text{\SSVM}}$ may suffer from not taking full advantage of
the accurate information about the confusion. 
%Therefore, our results act as bottom line
%comparison on how \UMA might perform. In practice though, the ability to
%directly use the confusion matrix without additional knowledge may allow to
%solve new problems and should not be overlooked.

\subsubsection{Datasets}
Our simulations are conducted on three different datasets. Each one with
different features. For the sake of reproducibility, we used datasets
that can be easily found on the \emph{UCI Machine learning repository}~\cite{Bache+Lichman:2013}.
Moreover, these datasets correspond to tasks for which generating a complete,
labelled, training set is typically costly because of the necessity of human
supervision and subject to classification noise. The datasets used and their main features are as follows.

\paragraph{Optical Recognition of Handwritten Digits.}
This well-known dataset is composed of $8\times 8$  grey-level images of
handwritten digits, ranging from $0$ to $9$. The dataset is composed of $3,823$
images of $64$ features for training, and $1,797$ for the test phase. We set $m$ to $10$
for this dataset, which means that $g$ is learned from $100$ examples only. 
$\trainingset_{\text{conf}}$ is a sampling of $5\%$ of $\trainingset$. The
classes are evenly distributed (see Figure~\ref{fig:handwritten}). \changeok{We
make use of a Gaussian kernel, and in the cases of the Perceptron and $\uma$ we
use a Kernel-\pca (see Section~\ref{sec:umakernels}) to (nonlinearly)
reduce the dimensionality of the data to $640$.}{We handle the nonlinearity through
the use of a Gaussian kernel-\pca (see section~\ref{sec:umakernels}) to
project the data onto a feature space of dimension $640$.}

\paragraph{Letter Recognition.}
The Letter Recognition dataset is another well-known pattern recognition
dataset. The images of the letters are summarized into a vector of $16$ attributes, which
correspond to various primitives computed on the raw data. With $20,000$
examples, this dataset is much larger than the previous
one. As for the Handwritten Digits dataset, the examples are evenly spread across the $26$ classes (see Figure~\ref{fig:letter}). We
uniformly select $15,000$ examples for training and the remaining $5,000$ are
used for test. We set $m$ to $50$ as it seems that smaller values do not
yield usable confusion matrices.
We again sample $5\%$ of the dataset to form $\trainingset_{\text{conf}}$
\changeok{and use a Gaussian kernel; as before, a kernel-\pca is performed for
the Perceptron and \uma to (nonlinearly) reduce the dimension of the data to
$1,600$.}{and use, as before, a Gaussian kernel-based Kernel-\pca to (nonlinearly)
expand the dimension of the data to
$1,600$.}

\paragraph{Reuters.}
The Reuters dataset is a nearly linearly-separable document categorization
dataset of more than $300,000$ instances of nearly $47,000$ features each. For
size reasons we restrict ourselves to roughly $15,000$ examples for training,
and $15,000$ other for test.
It occurs that some classes are so underrepresented that they are flooded by
the noise process and/or do not appear in $\trainingset_{\text{conf}}$, which may
lead to a non-invertible confusion matrix. We therefore restrict the dataset
to the $9$ largest classes. One might wonder whether doing so erases class imbalance.
This is not the case as, even this way, the least represented class accounts for
roughly $500$ examples while this number reaches nearly $4,000$ for the most
represented one (see Figure~\ref{fig:reuters}).
Actually, these $9$ classes represent more than $70$ percent of the dataset,
reducing the training and test sets to approximately $11,000$ examples each.
We do not use any kernel for this dataset, the data being already
near to linearly-separable.
Also, we sample $\trainingset_{\text{conf}}$ on $5\%$ of the training set and we
set $m = 20$.

\begin{figure}
  \centering
  \subfigure[Handwritten Digits]
{
%    \centering
    \includegraphics[width=0.30\textwidth]{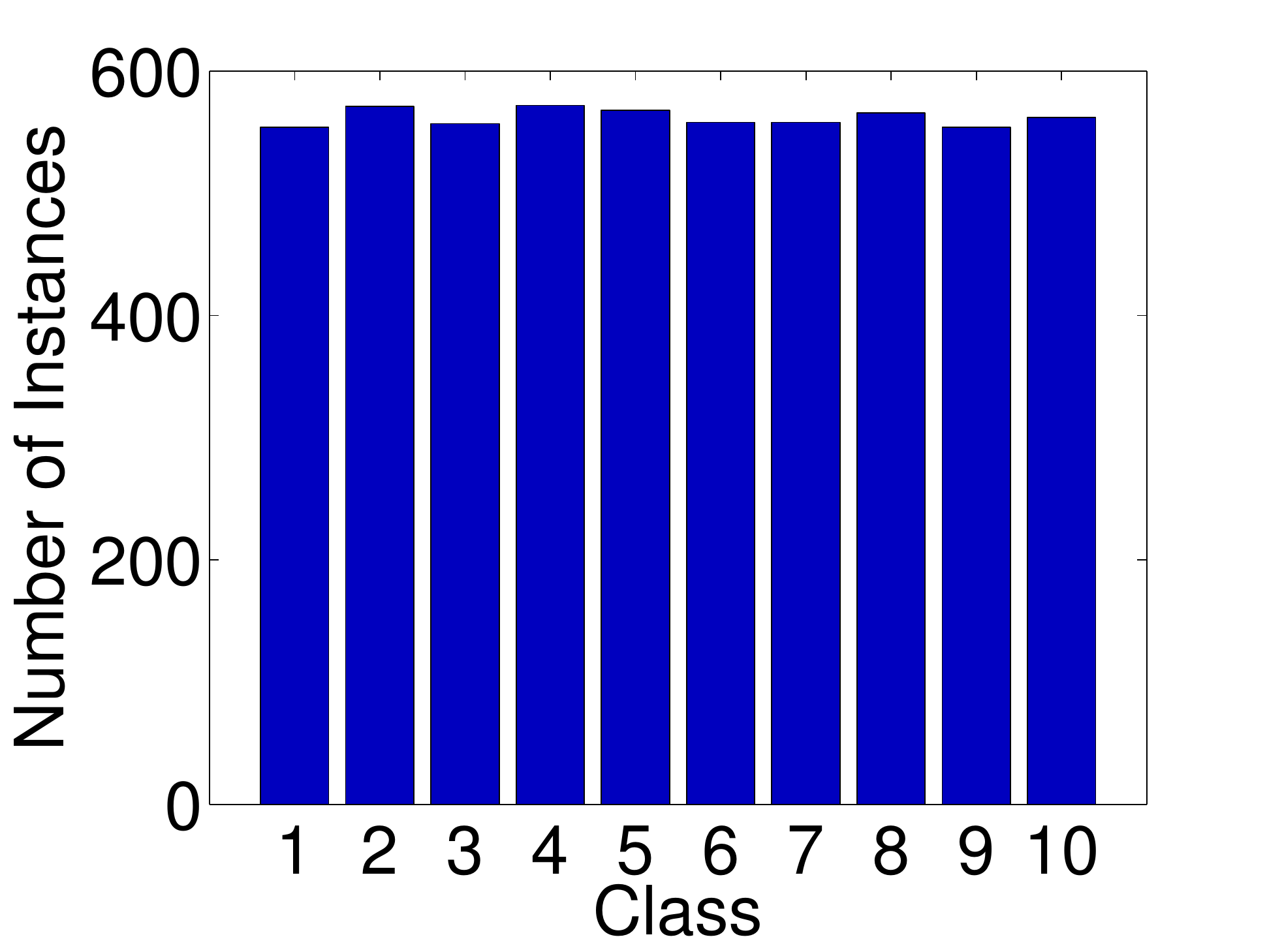}
    \label{fig:handwritten}
} 
\subfigure[Letter Recognition]
{
    \includegraphics[width=0.30\textwidth]{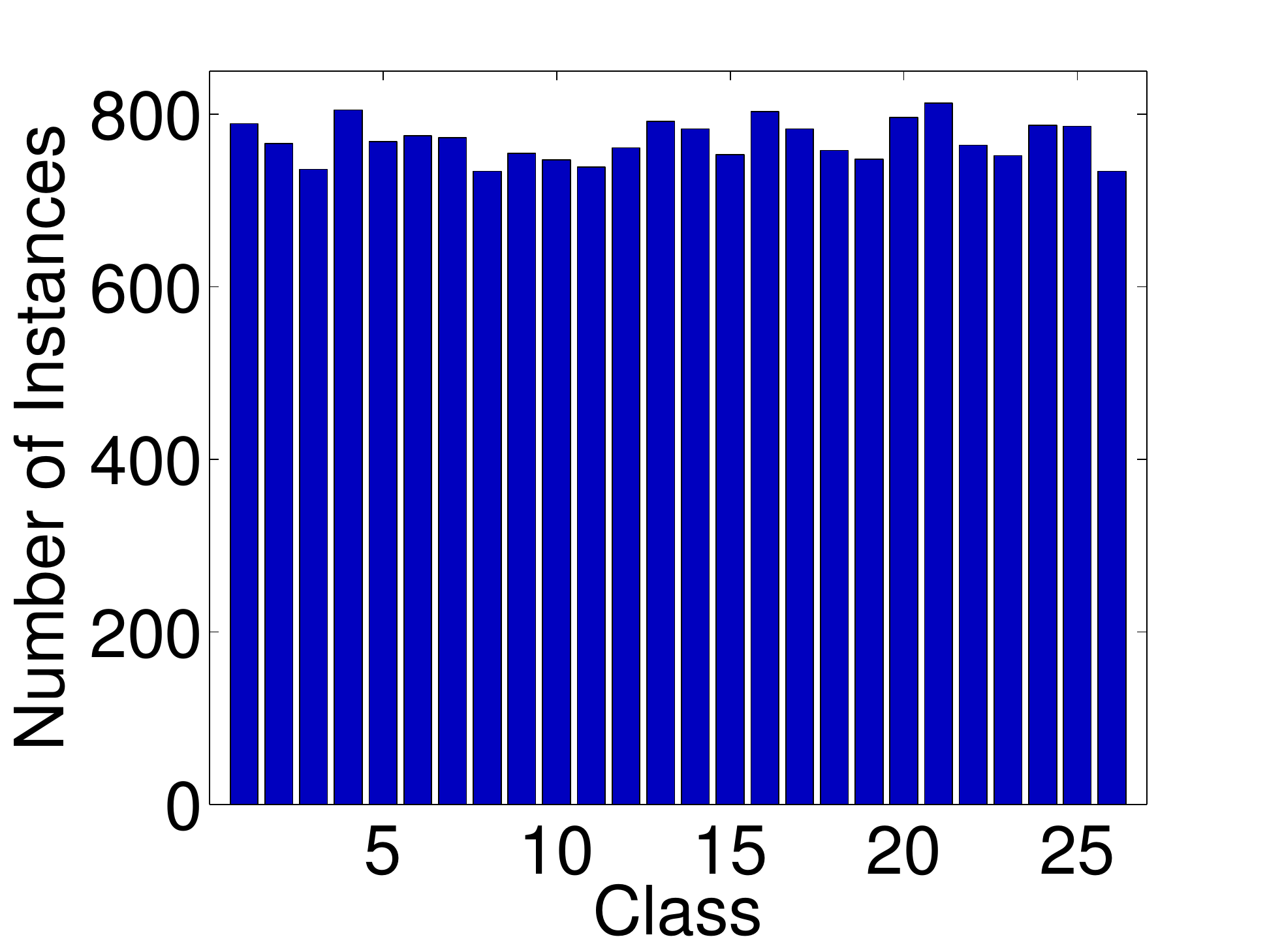}
    \label{fig:letter}
}
\subfigure[Reuters]
{
    \includegraphics[width=0.30\textwidth]{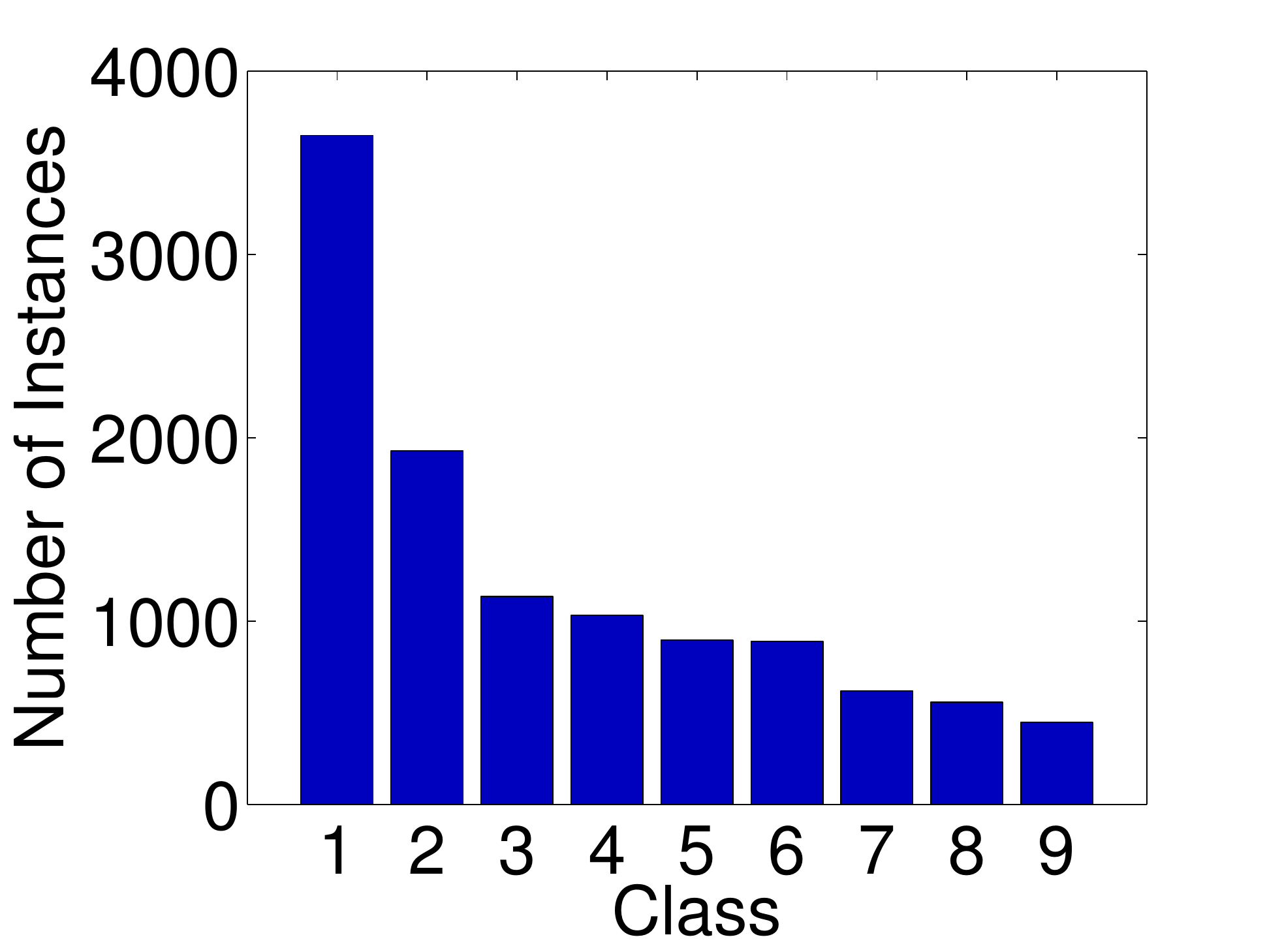}
    \label{fig:reuters}
}
  \caption{Class distribution for the three datasets.}
\end{figure}

\subsubsection{Results}

Table~\ref{fig:tab} presents the misclassification error rates
averaged on $10$ runs. Keep in mind that we have not conducted a very thorough optimization of the
hyper-parameters as the point here is essentially to compare \uma with the
other algorithms. 
\changeok{}{Additionally, we also report the error rates of $f_{\text{\SSVM}}$
when trained on the kernelized data with all dimensions, that is the kernelized
data before we project them onto their $D$ principal components. Because the projection
step is indeed unbecessary with \SSVM, this will give us insights on the error
due to the Kernel-\pca step.} 
Comparing the first and the last
columns of Table~\ref{fig:tab}, it appears that \uma always induces a slight
performance gain, {\em i.e.} a decrease of the misclassification rate, with respect to $f_{\bfy}$.

From the second and third columns of Table~\ref{fig:tab}, it is clear that
the reduced number of examples available to $f_{\text{conf}}$ induces a drastic
increase in the misclassification rate with respect to $f_{\text{full}}$ which
is allowed to use the totality of the dataset during the training phase.

Comparing \uma and $f_{\text{conf}}$ in Table~\ref{fig:tab} (fifth and second
columns), we observe that \uma
achieves lower misclassification rates on the Handwritten Digits and Letter
Recognition datasets but a higher misclassification rate on  Reuters. Although this
 is likely related to the strong class imbalance in the dataset. Indeed, some classes are overly represented, accounting for
the vast majority of the whole dataset (see Fig. \ref{fig:reuters}). Because
$\trainingset_{\text{conf}}$ is uniformly sampled from the main dataset,  
$f_{\text{conf}}$ is trained with a lot of examples from the overrepresented 
classes and therefore it is very effective, in the sense that it
achieves a low misclassification rate, for these
overrepresented classes; this, in turn, induces a (global) low misclassification rate, as possibly high misclassification rates on 
underrepresented classes are countervailed by theirs accounting for a small portion of the data. On the other
hand, because of this disparity in class representation, the slightest error in
the confusion matrix, granted it involves one of these overrepresented classes,
may lead to a significant increase of the
misclassification rate. In this regard, \uma is strongly disadvantaged with respect to $f_{\text{conf}}$ on the
Reuters dataset and it is the cause of the reported results.

\changeok{The
\SSVM classifier performs better than our algorithm when a kernel is involved, it achieves
very good results with the Handwritten Digits dataset and better results than
$f_{\text{full}}$ for the Letter Recognition problem. These results may be
rather puzzling. It seems unlikely that they are related to the (non-)usage of
the confusion matrix since in that case we would observe some correlation
between the misclassification rates of $f_{\text{conf}}$ and $f_{\text{\SSVM}}$
across datasets which is clearly not the case (see, for example, the previous discussion about
the Reuters dataset). A tentative explanation revolves on the structural
specifics of \SSVM and, essentially, the large margin paradigm, that give
 an edge to $f_{\text{\SSVM}}$ over the other classifiers based on
Perceptrons. While lacking theoretical evidence for this, empirical results in
the literature tends to confirm this~\cite{}; particularly when kernels are involved, 
which backs our observations. Moreover, this would explain that even on noisy
data, $f_{\text{\SSVM}}$ is able to achieve better classification performances
than $f_{\text{full}}$  whereas the latter has access to the complete and
taintless dataset.}
{The error rates for the \SSVM and \uma
classifiers are close for the Reuters and Handwritten Digits datasets whereas
\uma has a clear advantage on the Letter Recognition problem. On the other
hand, note that we used the \SSVM method in conjunction with a
Kernel-\pca for the sake of comparison with \uma in
 its kernelized form. The last column of Table \ref{fig:tab} tends to confirm
 that this projection strategy  increase the error rate of
 $f_{\text{\SSVM}}$. Also, reminds that the value of $m$ does not impact the
 performances of $f_{\text{\SSVM}}$ but has a significant effect on \uma, even
 though \uma never uses these labelled data.
For instance, on the Reuters datasets, increasing $m$ from $20$ to $70$ reduces
\uma's error rate by nearly $0.1$ (see the error rates of Fig. \ref{fig:f5}
($m=70$) when the size of labelled data is close to $550$, that is $5\%$ of the
whole dataset).
Despite our efforts to keep $m$ as small as possible, we could not go under $m =
50$ for the Letter Recognition dataset without compromising the
invertibility of the confusion matrix. The simple fact that an unusually high
number of examples are required to simply learn a rough classifier asserts the
complexity of this dataset. Moreover, the fact that
$f_{\bfy}$ also outperforms $f_{\text{\SSVM}}$ implies that the labels fed
to \uma are already mostly correct, and, according to our working assumptions,
 this is the most favorable setting for \uma.}

%Still SVMs are more refined than
%Perceptrons, therefore $f_{\text{\SSVM}}$ benefits from its structural advantage.

\begin{table}[tb]
  \centering
  \begin{tabular}{lcccccc}
  \toprule
  Dataset & {\tt $f_{\bfy}$} & {\tt $f_{\text{conf}}$} & {\tt
  $f_{\text{full}}$} & {\tt $f_{\text{\SSVM}}$} & \uma &
  {\tt $f_{\text{\SSVM}}$} (no K-\pca) \\ \midrule
  Handwritten Digits& $0.25$ & $0.21$ & $0.04$ & $0.15$ & $0.16$ & $0.07$ \\
  Letter Recognition & $0.35$ & $0.36$ & $0.23$ & $0.49$ & $0.33$& $0.18$ \\
  Reuters & $0.30$ & $0.17$ & $0.01$ & $0.22$ & $0.21$ & $0.22$  \\
  \bottomrule
  \end{tabular}
  \caption{Misclassification rates of different algorithms.}
  \label{fig:tab}
\end{table}

Nonetheless, the disparities between \uma and $f_{\text{conf}}$ deserve
more attention. Indeed, the same data are being used by both algorithms,
and one could expect more closeness in the results. To get a better
insight on what is occurring, we have reported
the evolution of the error rate of these two algorithms with respect to the
sampling size of $\trainingset_{\text{conf}}$ in Figure~\ref{fig:f5}. 
We can see that \uma is unaffected by the size of the sample, essentially
ignoring the possible errors in the confusion matrix on small samples. This
reinforces our previous results showing that \uma is robust to errors in the confusion matrix.
On the other hand, with the addition of more samples, the refinement of the
confusion matrix does not allow \uma to compete with the value of additional
(correctly) labelled data and eventually, when the size of
$\trainingset_{\text{conf}}$ grows, $f_{\text{conf}}$ performs better than \uma.
This points towards the idea that the aggregated nature of the confusion matrix
incurs some loss of relevant information for the classification
task at hand, and that a more accurate estimate of the confusion matrix, as
induced by, {\em e.g.}, the use of larger $\trainingset_{\text{conf}}$, may not
compensate for the information provided by additional raw data. 

\changeok{}{Building on this observation, we go a step further and replicate this
experiment for all of the three datasets; only this time we track the
performances of $f_{\text{\SSVM}}$ instead. The results are plotted on Figure~\ref{fig:ExpeRev2}.
 For the three datasets, we observe the same behavior as
before. Namely, \uma is able to maintain a low error rate even with a very small
size of $\trainingset_{\text{conf}}$. On the other hand, \uma does not benefit
as much as other methods from a large pool of labelled examples. In this case,
\uma quickly stabilizes while, to the contrary, the \SSVM method starts at a fairly high error rate
and keeps improving as more labelled examples are available.}

Beyond this,
it is important to recall that \uma never uses the labels of
$\trainingset_{\text{conf}}$ (those are only used to estimate the
confusion matrix, not the classifier---refer to Section \ref{sec:modus} for
the detailed learning protocol). While refining the estimation of $\confusion$
is undoubtedly useful, a direction toward substantial performance gains
should revolve around the combination of both this refined estimation of
$\confusion$ {\em and} the use of the correctly labelled training set
$\trainingset_{\text{conf}}$. This is a research subject on its own that we
leave for future work.

\begin{figure}[tb]
\centering
  \includegraphics[scale=0.5]{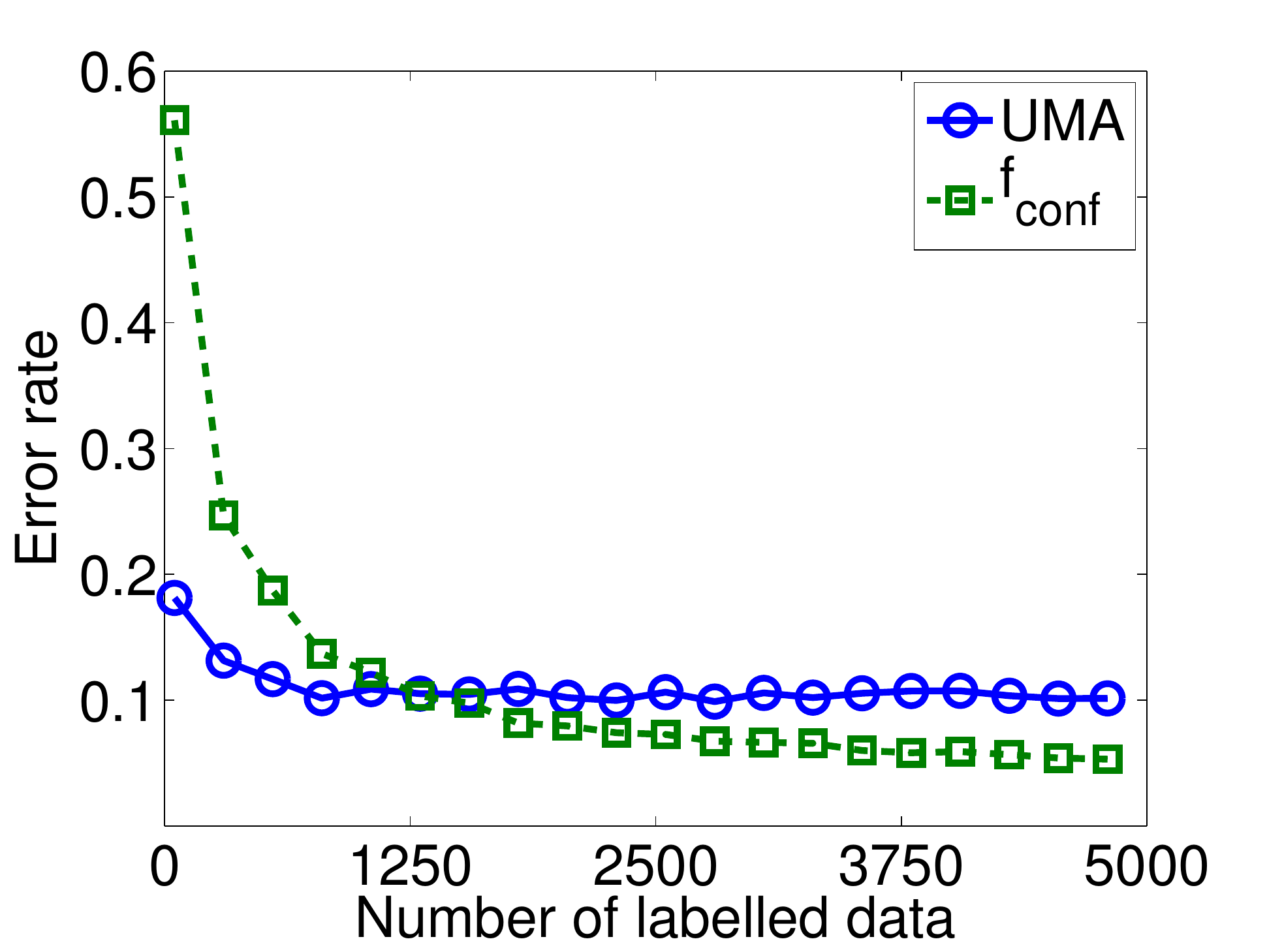}
  \caption{Error rate of \uma and $f_{\text{conf}}$ with respect to the
  sampling size. Reuters dataset with $m=70$ for the sake of figure's
  readability.}
  \label{fig:f5}
\end{figure}

\begin{figure}[tb]
  \centering
  \subfigure[Reuters dataset]
{
%    \centering
    \includegraphics[width=0.31\textwidth]{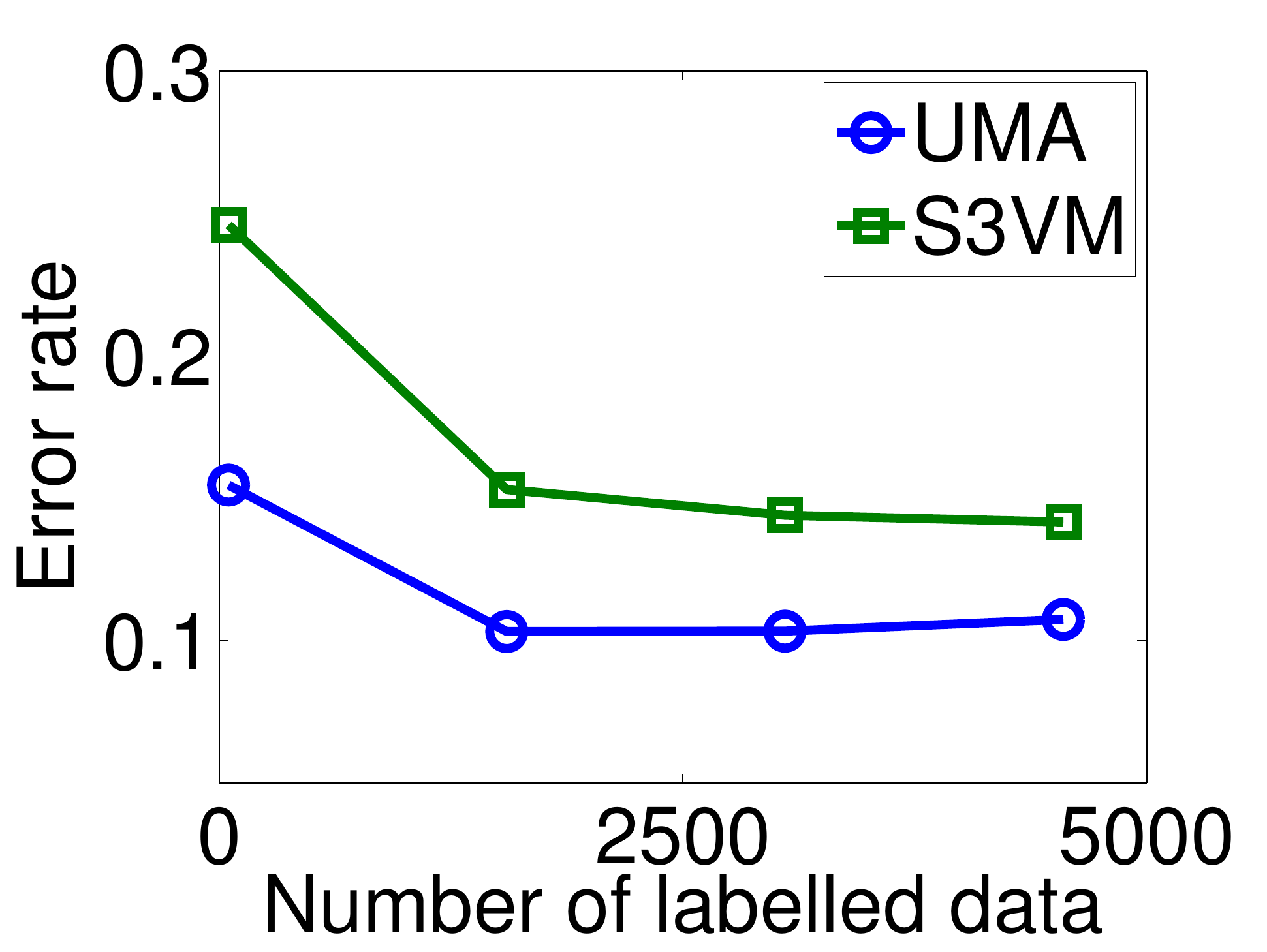}
} 
\subfigure[Digits dataset]
{
    \includegraphics[width=0.31\textwidth]{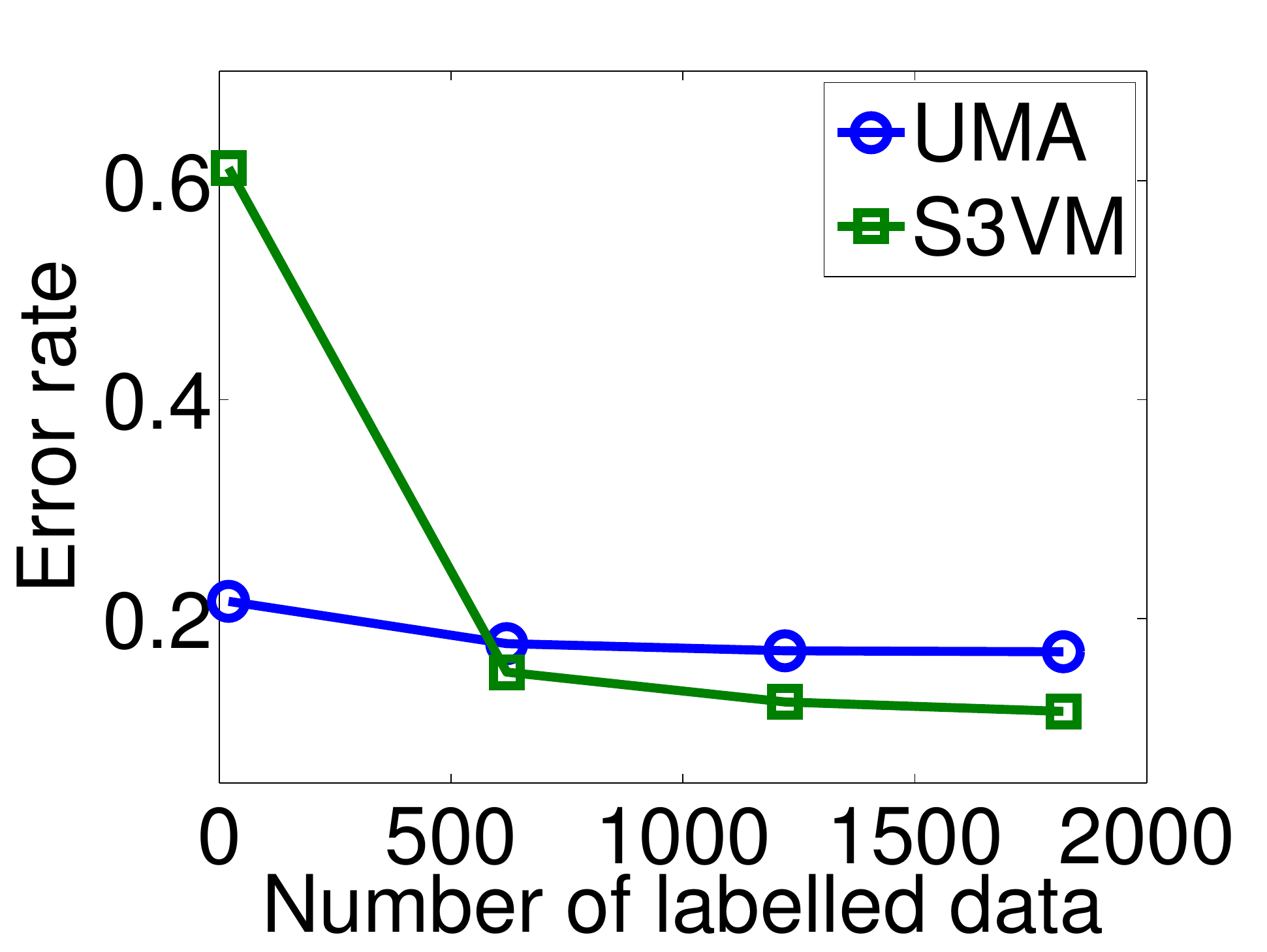}
}
\subfigure[Letter dataset]
{
    \includegraphics[width=0.31\textwidth]{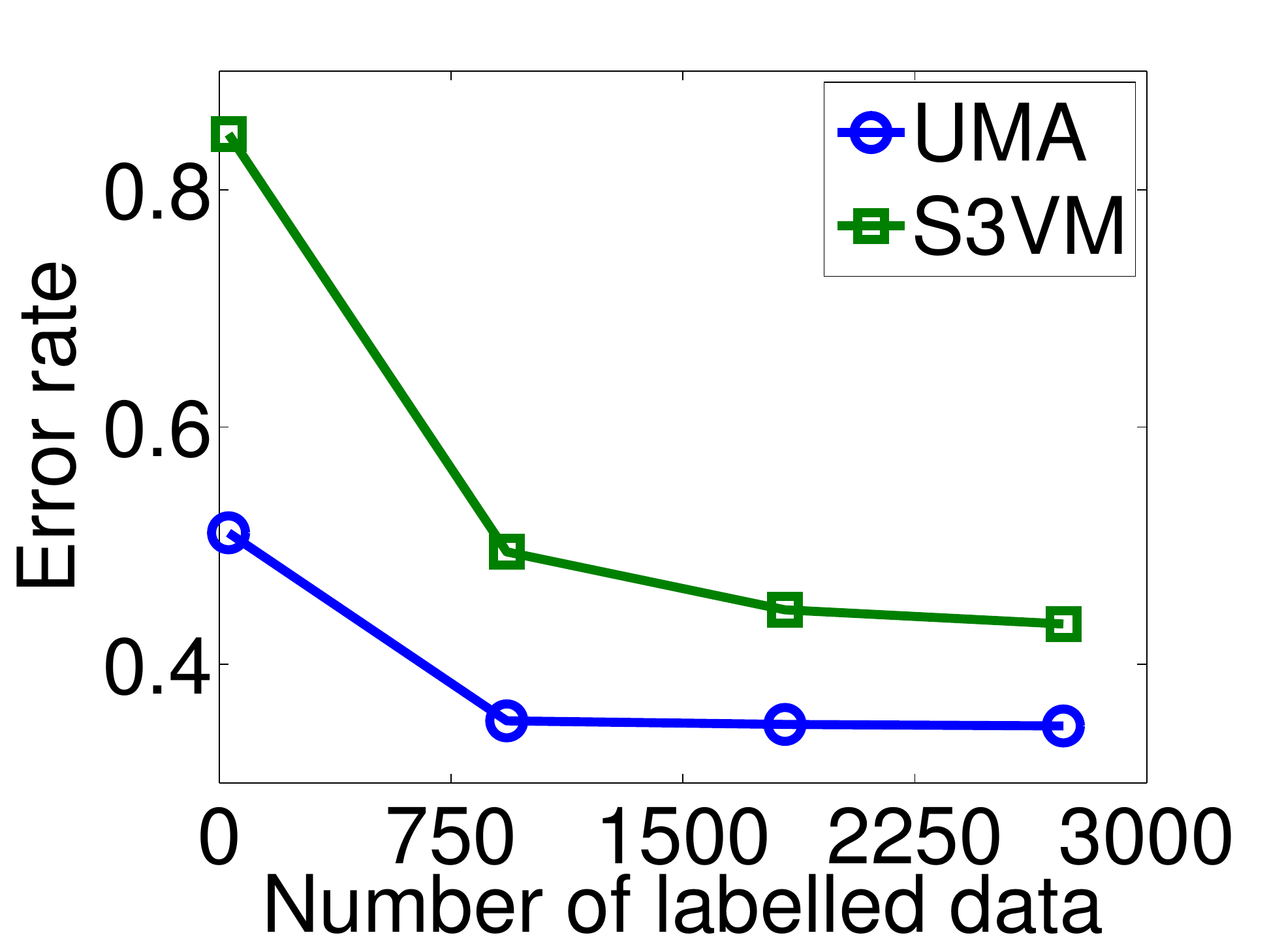}
}
  \caption{Error rates for the Reuter (left), optical digit recognition (center)
  and letter (right) datasets with respect to the size of
  $\trainingset_{\text{conf}}$. Average over $15$ runs. }
  \label{fig:ExpeRev2}
\end{figure}

All in all, the reported results advise us to prefer \uma over other available
methods when the amount of labelled data is particularly small, in addition, obviously, to the motivating case of
the present work where the training data are corrupted and the confusion matrix
is known. Also, another interesting finding we get is that even a rough
estimation of the confusion matrix is sufficient for \uma to behave well.

Finally, we investigate the impact of the selection strategy of $(p,q)$ on the
convergence speed of \uma (see Section \ref{sec:pqselect}).
We use three variations of \uma with different strategies
for selecting $(p,q)$ (error, confusion, and random) and monitor
each one along the learning process on the Reuters dataset. The error and
confusion strategies are described in Section \ref{sec:pqselect} and the random
strategy simply selects $p$ and $q$ at random.

\begin{figure}
  \centering
  \subfigure[Error rate]
{
%    \centering
    \includegraphics[width=0.48\textwidth]{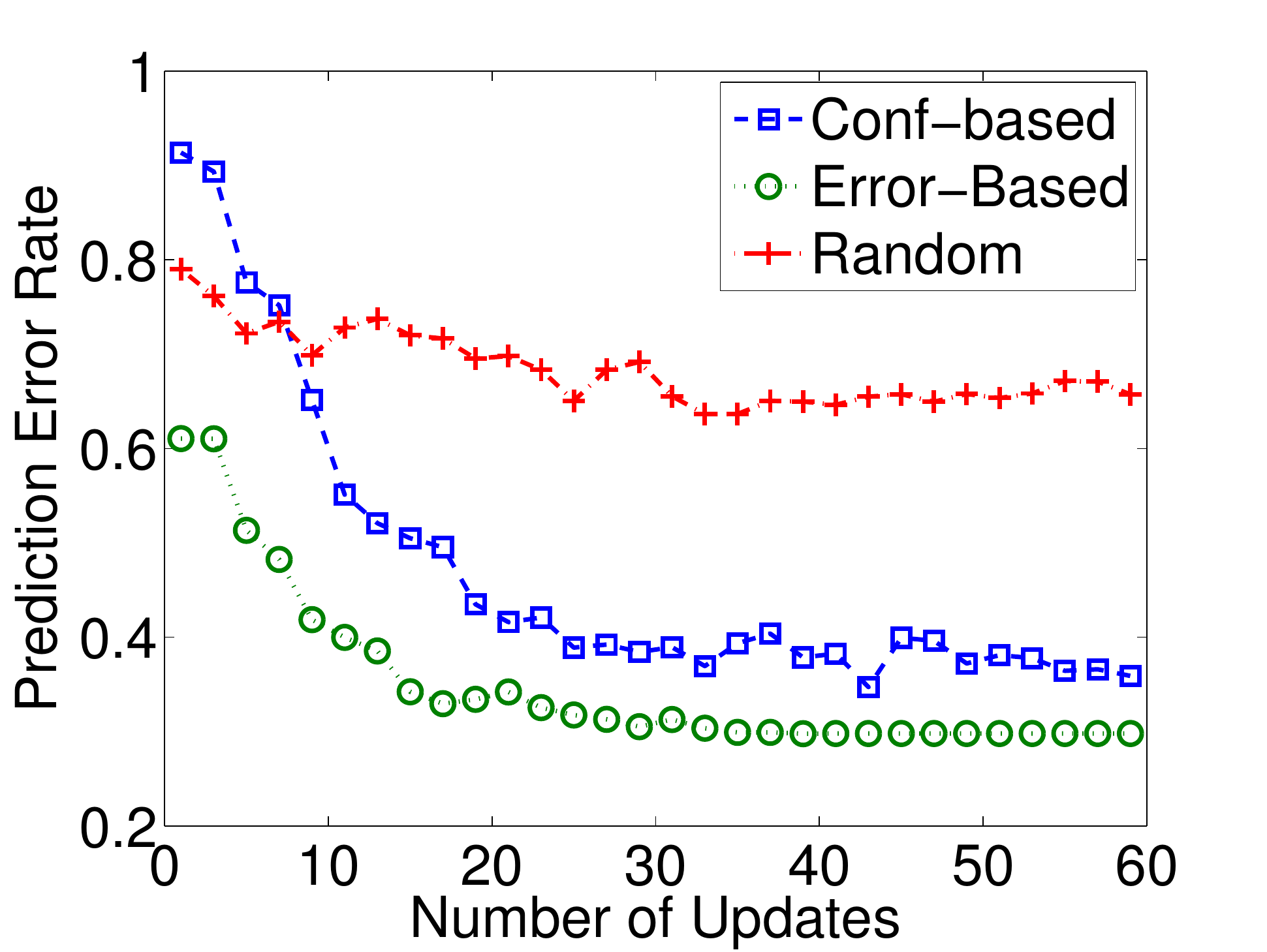}
} 
\subfigure[Confusion rate]
{
    \includegraphics[width=0.48\textwidth]{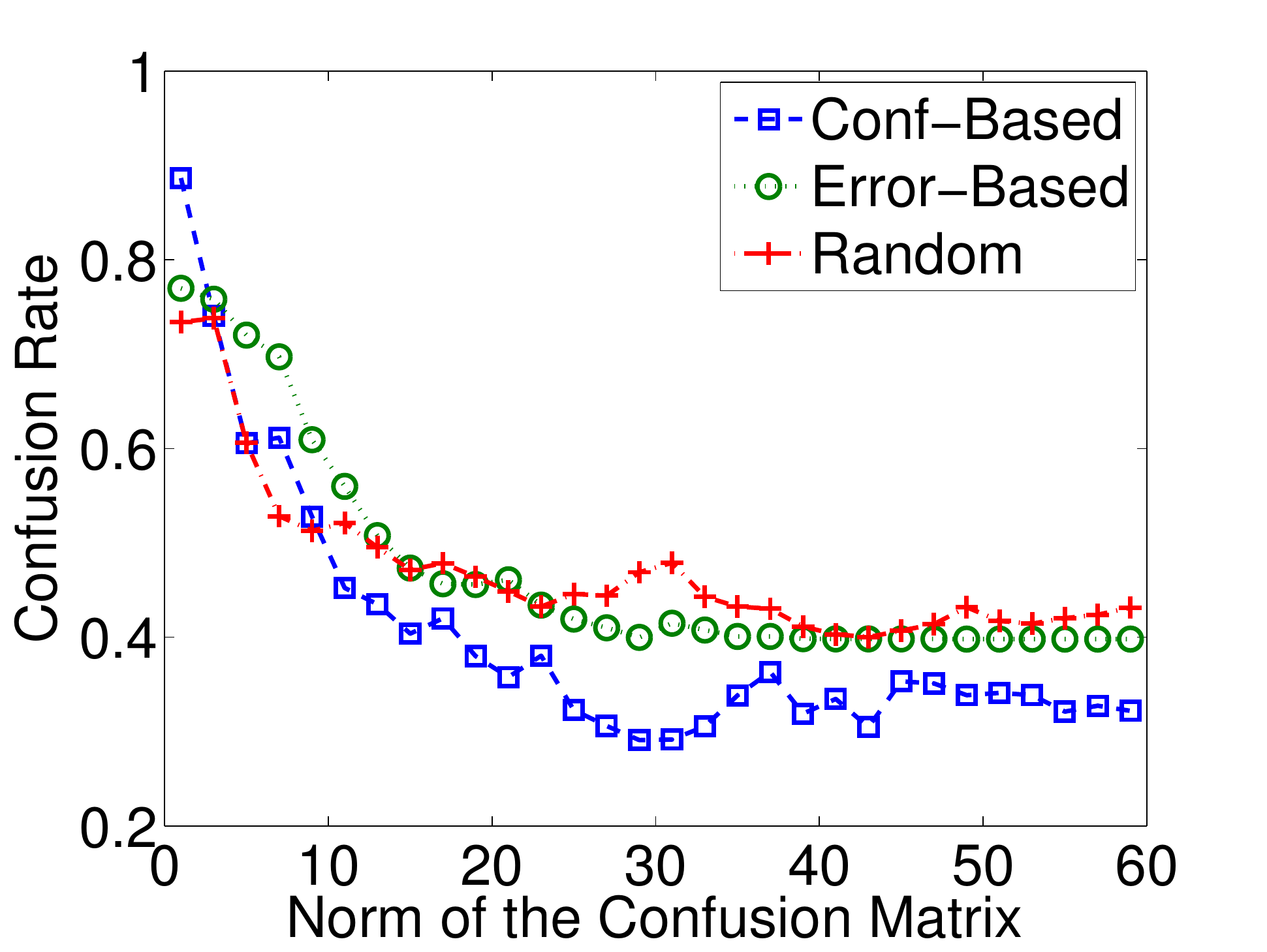}
}
  \caption{Error and confusion risk on Reuters dataset
  with various update strategies.}
  \label{fig:f3}
\end{figure}

From Figure \ref{fig:f3}, which reports the misclassification rate and the
confusion rate along the iterations, we observe that both performance
measures evolve similarly, attaining a stable state around the $30$th
iteration. The best strategy depends on the performance measure used, even
though regardless of the performance measure used, we observe that the
random selection strategy leads to a predictor that does not achieve the best
performance measure (there is always a curve beneath that of the random
selection procedure), which shows that it not an optimal selection strategy.

As one might expect, the confusion-based strategy
performs better than the error-based strategy
 when the confusion rate is retained as a performance measure,
while the converse holds when using the error rate. This
observation motivates us to thoroughly study the confusion-based
strategy in a near future as being able to propose methods robust to class
imbalance is a particularly interesting challenge of multiclass classification.

The plateau reached around the $30$th iteration may be puzzling,
since the studied dataset presents
no positive margin and convergence
is therefore not guaranteed.
One possible explanation for this is to see the Reuters dataset as linearly
separable problem corrupted by the effect of a noise process, which we call the
{\em intrinsic noise process} that has structural features `compatible' with
the classification noise. By this, we mean that there must be features of the
intrinsic noise such that, when additional classification noise is added, the
resulting noise that characterizes the data is similar to a classification
noise, or at least, to a noise that can be naturally handled by \uma.
Finding out the family of noise processes that can be combined with the
classification noise---or, more generally, the family of noise processes
themselves---without hindering the effectiveness of \uma is one research
direction that we aim to explore in a near future.

%!TEX root=unconfused.tex
\section{Conclusion} \label{sec:Concl}

In this paper, we have proposed a new algorithm, \uma---for Unconfused Multiclass Additive algorithm---to cope
with noisy training examples in multiclass linear problems. As 
its name indicates, it is a learning procedure that extends the (ultraconservative)
additive multiclass algorithms proposed by \cite{crammer03ultraconservative}; to
handle the noisy datasets, it only requires the information about the confusion
matrix that characterizes the mislabelling process.
This is, to the best of our knowledge, the first time
the confusion matrix is used as a way to handle noisy label
in multiclass problems.

One of the core ideas behind \uma, namely, the computation of the update vector $\xup{pq}$, is
not tied to the additive update scheme. Thus, as long as the assumption of
linear separability holds, the very same idea can be used to
render a wide variety of algorithms robust to noise by iteratively generating 
a noise-free training set with the consecutive values of $\xup{pq}$. 
Although, every computation of a new $\xup{pq}$ requires learning a new
classifier to start with. This may eventually incur prohibitive computational
costs when applied to batch methods (as opposed to online methods) which are
designed to process the entirety of the dataset at once.~\footnote{Nonetheless, from a purely theoretical point of view, \uma makes at
most $O(1/\margin^2)$ mistakes (see proposition \ref{prop:convergence}) and
computing $\xup{pq}$ can be done in $O(n)$ time. Therefore, polynomial batch
methods do not suffer much from this as their overall execution time is still
polynomial.}

\uma takes advantage of the online scheme of additive algorithms and avoids this
problem completely.
Moreover, additive algorithms are designed to directly handle
multiclass problem rather than having recourse to a bi-class mapping. The end-results of this are
tightened theoretical guarantees and a convergence rate that does not depend of
$Q$, the number of classes.
Besides, \uma can be directly used with any 
additive algorithms, allowing to handle noise with
multiple methods without further computational burden.

While we provide sample complexity analysis, it should be noted that
 a tighter bound can
be derived with specific multiclass tools, such as
the Natarajan's dimension (see \cite{DanielySBS11} 
for example), which allow to better specify the expressiveness of a multiclass classifier.
However, this is not the main focus of this paper and our results are based on
simpler tools.

To complement this work, we want to investigate
a way to properly tackle near-linear problems (such as Reuters). 
As for now the algorithm already does a very good jobs due to its
noise robustness. However more work has to be done to derive a
proper way to handle cases where a perfect classifier does not exist.
We think there
are great avenues for interesting research in this domain with an algorithm
like \uma and we are curious to see how this present work may carry
over to more general problems.

\paragraph{Acknowledgments.} The authors would like to thank the reviewers for their feedback and invaluable comments. This work is partially supported by the Agence Nationale de la Recherche (ANR), project GRETA 12-BS02-004-01. We would like to thank
the anonymous reviewers for their insightful and extremely valuable feedback on earlier versions of this paper.

\bibliographystyle{spmpsci}
\bibliography{unconfused}

\begin{thebibliography}{10}
\providecommand{\url}[1]{{#1}}
\providecommand{\urlprefix}{URL }
\expandafter\ifx\csname urlstyle\endcsname\relax
  \providecommand{\doi}[1]{DOI~\discretionary{}{}{}#1}\else
  \providecommand{\doi}{DOI~\discretionary{}{}{}\begingroup
  \urlstyle{rm}\Url}\fi

\bibitem{bach02kernel}
Bach, F.R., Jordan, M.I.: Kernel {I}ndependent {C}omponent {A}nalysis.
\newblock Journal of Machine Learning Research \textbf{3}, 1--48 (2002)

\bibitem{Bache+Lichman:2013}
Bache, K., Lichman, M.: {UCI} machine learning repository (2013).
\newblock \urlprefix\url{http://archive.ics.uci.edu/ml}

\bibitem{Bennett}
Bennett, K.P., Demiriz, A.: Semi-supervised support vector machines.
\newblock In: Advances in Neural Information Processing Systems, pp. 368--374.
  MIT Press (1998)

\bibitem{blanchard08finite}
Blanchard, G., Zwald, L.: Finite-dimensional projection for classification and
  statistical learning.
\newblock IEEE Transactions on Information Theory \textbf{54}(9), 4169--4182
  (2008)

\bibitem{block62perceptron}
Block, H.: The perceptron: a model for brain functioning.
\newblock Reviews of Modern Physics \textbf{34}, 123--135 (1962)

\bibitem{blum96polynomialtime}
Blum, A., Frieze, A.M., Kannan, R., Vempala, S.: A {P}olynomial-{T}ime
  {A}lgorithm for {L}earning {N}oisy {L}inear {T}hreshold {F}unctions.
\newblock In: Proc. of 37th {IEEE} Symposium on Foundations of Computer
  Science, pp. 330--338 (1996)

\bibitem{bylander94learning}
Bylander, T.: {Learning Linear Threshold Functions in the Presence of
  Classification Noise}.
\newblock In: Proc. of 7th Annual Workshop on Computational Learning Theory,
  pp. 340--347. ACM Press, New York, NY, 1994 (1994)

\bibitem{cohen97learning}
Cohen, E.: Learning {N}oisy {P}erceptrons by a {P}erceptron in {P}olynomial
  {T}ime.
\newblock In: Proc. of 38th {IEEE} Symposium on Foundations of Computer
  Science, pp. 514--523 (1997)

\bibitem{crammer06online}
Crammer, K., Dekel, O., Keshet, J., Shalev-Shwartz, S., Singer, Y.: Online
  passive-aggressive algorithms.
\newblock JMLR \textbf{7}, 551--585 (2006)

\bibitem{crammer03ultraconservative}
Crammer, K., Singer, Y.: Ultraconservative online algorithms for multiclass
  problems.
\newblock Journal of Machine Learning Research \textbf{3}, 951--991 (2003)

\bibitem{cristianini00introduction}
Cristianini, N., Shawe-Taylor, J.: An {I}ntroduction to {S}upport {V}ector
  {M}achines and other {K}ernel-{B}ased {L}earning {M}ethods.
\newblock Cambridge University Press (2000)

\bibitem{DanielySBS11}
Daniely, A., Sabato, S., Ben-David, S., Shalev-Shwartz, S.: Multiclass
  learnability and the {ERM} principle.
\newblock Journal of Machine Learning Research - Proceedings Track \textbf{19},
  207--232 (2011)

\bibitem{Dekel05theforgetron}
Dekel, O., Shalev-shwartz, S., Singer, Y.: The forgetron: A kernel-based
  perceptron on a fixed budget.
\newblock In: In Advances in Neural Information Processing Systems 18, pp.
  259--266. MIT Press (2005)

\bibitem{DevroyeGL}
Devroye, L., Gy{\"{o}}rfi, L., Lugosi, G.: {A Probabilistic Theory of Pattern
  Recognition}.
\newblock Springer (1996)

\bibitem{DrineasKM06SIAM}
Drineas, P., Kannan, R., Mahoney, M.W.: Fast {Monte Carlo} algorithms for
  matrices ii: computing a low rank approximation to a matrix.
\newblock SIAM Journal on Computing \textbf{36}(1), 158--183 (2006)

\bibitem{drineas05nystrom}
Drineas, P., Mahoney, M.W.: On the {Nystr\"om} method for approximating a gram
  matrix for improved kernel-based learning.
\newblock Journal of Machine Learning Research \textbf{6(Dec)}, 2153--2175
  (2005)

\bibitem{freund99large}
Freund, Y., Schapire, R.E.: Large {M}argin {C}lassification {U}sing the
  {P}erceptron {A}lgorithm.
\newblock Machine Learning \textbf{37}(3), 277--296 (1999)

\bibitem{cristianini98kernel-adatron}
Friess, T., Cristianini, N., Campbell, N.: The {K}ernel-{A}datron {A}lgorithm:
  a {F}ast and {S}imple {L}earning {P}rocedure for {S}upport {V}ector
  {M}achines.
\newblock In: J.~Shavlik (ed.) Machine Learning: Proc. of the $15^{th}$ Int.
  Conf. Morgan Kaufmann Publishers (1998)

\bibitem{Kakade2008Banditron}
Kakade, S.M., Shalev-Shwartz, S., Tewari, A.: Efficient bandit algorithms for
  online multiclass prediction.
\newblock In: Proceedings of the 25th International Conference on Machine
  Learning, ICML '08, pp. 440--447. ACM, New York, NY, USA (2008)

\bibitem{kearns94introduction}
Kearns, M.J., Vazirani, U.V.: An {I}ntroduction to {C}omputational {L}earning
  {T}heory.
\newblock MIT Press (1994)

\bibitem{LoucheR13ACML}
Louche, U., Ralaivola, L.: Unconfused ultraconservative multiclass algorithms.
\newblock In: JMLR Workshop \& Conference Proc. 29, (Proc. of ACML 13), pp.
  309--324 (2013)

\bibitem{minsky69perceptrons}
Minsky, M., Papert, S.: Perceptrons: an {I}ntroduction to {C}omputational
  {G}eometry.
\newblock MIT Press (1969)

\bibitem{novikoff63convrgence}
Novikoff, A.: On convergence proofs for perceptrons.
\newblock In: Proc. of the Symposium on the Mathematical Theory of Automata,
  Vol. 12, pp. 615--622 (1963)

\bibitem{Ralaivola12}
Ralaivola, L.: Confusion-based online learning and a passive-aggressive scheme.
\newblock In: NIPS, pp. 3293--3301 (2012)

\bibitem{RalaivolaFGBD11}
Ralaivola, L., Favre, B., Gotab, P., Bechet, F., Damnati, G.: Applying
  multiclass bandit algorithms to call-type classification.
\newblock In: ASRU, pp. 431--436 (2011)

\bibitem{schoelkopf02learning}
Sch{\"o}lkopf, B., Smola, A.J.: Learning with {K}ernels, {S}upport {V}ector
  {M}achines, {R}egularization, {O}ptimization and {B}eyond.
\newblock MIT University Press (2002).
\newblock \urlprefix\url{http://www.learning-with-kernels.org}

\bibitem{stempfel07learning}
Stempfel, G., Ralaivola, L.: Learning kernel perceptron on noisy data and
  random projections.
\newblock In: In Proc. of Algorithmic Learning Theory (ALT 07) (2007)

\bibitem{conf/cvpr/TakerkartR11}
Takerkart, S., Ralaivola, L.: {MKPM:} a multiclass extension to the kernel
  projection machine.
\newblock In: CVPR, pp. 2785--2791 (2011).
\newblock
  \urlprefix\url{http://dblp.uni-trier.de/db/conf/cvpr/cvpr2011.html#TakerkartR11}

\bibitem{valiant84theory}
Valiant, L.: A theory of the learnable.
\newblock Communications of the ACM \textbf{27}, 1134--1142 (1984)

\bibitem{williams01nystrom}
Williams, C.K.I., Seeger, M.: {Using the {Nystr\"om} Method to Speed Up Kernel
  Machines}.
\newblock In: Advances in Neural Information Processing Systems 13, pp.
  682--688. MIT Press (2001)

\end{thebibliography}

\appendix
\normalsize
%!TEX root=unconfused.tex
\section{Double sample theorem}\label{apd:first}

\begin{proof}[Proposition \ref{prop:zpqerror}]

For a fixed pair $(p,q) \in \outputspace^2$,  we consider
the family of functions
$$ \bigF_{pq} \doteq \lbrace f: f(\V{x}) \doteq 
\dotProd{\V{w}_q - \V{w}_p}{x}: \V{w}_p, \V{w}_q \in \mathcal{B}^d
\rbrace $$ 
where $\mathcal{B}^d$ is a $d$-dimensional unit ball.
For each $f \in \bigF_{pq}$ define the corresponding ``loss'' function
$$l^{f}(\V{x}) \doteq l(f(\V{x})) \doteq 2 - f(\V{x}).$$

Strictly speaking, $l^{f}(\V{x})$ is not a loss as it does not take $y$ into
account, nonetheless it does play the same role in the following proof than a
regular loss in the regular double-sampling proof. One way to think of it is as
the loss of a problem for which we do not care about the observed labels but
instead we want to classify points into a predetermined class---in this case
$q$.

Clearly, $\bigF_{pq}$ is a subspace of affine functions, thus 
$\Pdim(\bigF_{pq}) \leq (d + 1)$, where $\Pdim(\bigF_{pq})$ is the
pseudo-dimension of $\bigF_{pq}$. Additionally, $l$ is Lipschitz
in its first argument with a Lipschitz factor of $L \doteq 1$. Indeed
$\forall y, y_1, y_2, \in \outputspace:
 \vert l(y_1, y) - l(y_2, y) \vert = \vert y_1 - y_2 \vert$.

Let $\UKDistri_{pq}$ be any distribution over $\inputspace \times
\outputspace$ and $T \in (\inputspace \times \outputspace)^m$ 
such that $T \sim \UKDistri_{pq}^m$, then define the 
\emph{empirical loss} $\err_T^l[f] \doteq \frac{1}{m}
\sum_{\V{x}_i \in T} l(\V{x}_i, y_i)$ and the
\emph{expected loss} $\err_{\UKDistri}^l [f] \doteq
 \Ex{\UKDistri}{l(\V{x},y)}$

The goal here is to prove that

\begin{align} \label{eq:result}
  \proba_{T \sim \UKDistri_{pq}^m} \left(
\sup_{f \in \bigF_{pq}} \vert \err_{\UKDistri}^l [f]
  - \err_{T}^l [f] \vert \geq \epsilon
 \right) \in \bigO \left(
 \left( \frac{8}{\epsilon} \right)^{(d+1)}
e^{ {m\epsilon^2}/_{128} } \right)
\end{align}

\begin{proof}[Proof of \eqref{eq:result}]
We start by noting that $l(y_1, y_2) \in [0, 2]$ and then proceed
with a classic $4$-step double sampling proof. Namely:
\paragraph{\bf Symmetrization.}
We introduce a \emph{ghost} sample $T' \in (\inputspace \times
\outputspace)^m$, $T' \sim \UKDistri_{pq}^m$ and show that for
$f^{\text{bad}}_T$ such that $ \vert \err_{\UKDistri_{pq}}^l
[f^{\text{bad}}_T] - \err_T^l [f^{\text{bad}}_T] \vert \geq \epsilon
$ then 

$$ \proba_{T' \vert T} \left(
\left| \err_{T'}^l [f^{\text{bad}}_T] - \err_{\UKDistri_{pq}}^l 
[f^{\text{bad}}_T] \right| \leq \frac{\epsilon}{2}
\right) \geq \frac{1}{2},$$
as long as $m\epsilon^2 \geq 32.$

It follows that 
\begin{align*}
& \proba_{(T,T') \sim \UKDistri_{pq}^m \times \UKDistri_{pq}^m}
\left( \sup_{f \in \bigF_{pq}} \vert \err_T^l [f] - \err_{T'}^l [f]
  \vert \geq \frac{\epsilon}{2}
\right) \\
&~ \geq \proba_{T \sim \UKDistri_{pq}^m} 
\left( \vert \err_{T}^l [f^{\text{bad}}_T] - \err_{\UKDistri_{pq}}^l
  [f^{\text{bad}}_T] \vert \geq \epsilon
\right) \times
\proba_{T' \vert T} \left(
\left| \err_{T'}^l [f^{\text{bad}}_T] - \err_{\UKDistri_{pq}}^l 
[f^{\text{bad}}_T] \right| \leq \frac{\epsilon}{2}
\right) \\
&~  \geq \frac{1}{2} \proba_{T \sim \UKDistri_{pq}^m} 
\left( \vert \err_{T}^l [f^{\text{bad}}_T] - \err_{\UKDistri_{pq}}^l
  [f^{\text{bad}}_T] \vert \geq \epsilon
\right) \\
&~  = \frac{1}{2} \proba_{T \sim \UKDistri_{pq}^m} 
\left( \sup_{f \in \bigF_{pq}} \vert \err_{T}^l [f] - \err_{\UKDistri_{pq}}^l
  [f] \vert \geq \epsilon
\right) \tag{By definition of $f^{\text{bad}}_T$}
\end{align*}

Thus upper bounding the desired probability by
\begin{align}
2 \times \proba_{(T,T') \sim \UKDistri_{pq}^m \times \UKDistri_{pq}^m}
\left( \sup_{f \in \bigF_{pq}} \vert \err_T^l [f] - \err_{T'}^l [f]
  \vert \geq \frac{\epsilon}{2}
\right) \label{eq:sndProba}
\end{align}

\paragraph{\bf Swapping Permutations.}
Let define $\Gamma_{m}$ the set of all
permutations that swap one or more elements of
$T$ with the corresponding element of $T'$ (i.e.
the $i$th element of $T$ is swapped with 
the $i$th element of $T'$). It is quite 
immediate that $\vert \Gamma_{m} \vert = 2^m$.
For each permutation $\sigma \in \Gamma_m$ we
note $\sigma(T)$ (resp. $\sigma(T')$) the set
originating from $T$ (resp. $T'$) from which
the elements have been swapped with $T'$
(resp. $T$) according to $\sigma$.

Thanks to $\Gamma_m$ we will be able to provide an
upper bound on \eqref{eq:sndProba}. Our starting
point is that $(T,T') \sim \UKDistri_{pq}^m \times
\UKDistri_{pq}^m$ then for any $\sigma \in \Gamma_m$, 
the random variable
$\sup_{f \in \bigF_{pq}} \vert \err_T^l [f] - \err_{T'}^l [f] \vert$
follows the same distribution as
$\sup_{f \in \bigF_{pq}} \vert \err_{\sigma(T)}^l [f] -
\err_{\sigma(T')}^l [f] \vert$.

Therefore:
\begin{align}
& \proba_{(T,T') \sim \UKDistri_{pq}^m \times \UKDistri_{pq}^m}
\left( \sup_{f \in \bigF_{pq}} \vert \err_T^l [f] - \err_{T'}^l [f]
  \vert \geq \frac{\epsilon}{2}
\right) \notag \\
& ~ = \frac{1}{2^m} \sum_{\sigma \in \Gamma_m}
\proba_{T,T' \sim \UKDistri_{pq}^m \times \UKDistri_{pq}^m} \left(
\sup_{f \in \bigF_{pq}} \vert \err_{\sigma(T)}^l [f] - 
\err_{\sigma(T')}^l [f] \vert \geq \frac{\epsilon}{2}
\right) \notag \\
&~ = \Ex{(T,T') \sim \UKDistri_{pq}^m \times \UKDistri_{pq}^m}
{\frac{1}{2m} \sum_{\sigma \in \Gamma_m} \indicator{
\sup_{f \in \bigF_{pq}} \vert \err_{\sigma(T)}^l [f] -
\err_{\sigma(T')}^l [f] \vert \geq \frac{\epsilon}{2}
}} \notag \\
& ~ \leq \sup_{(T,T') \in (\inputspace \times \outputspace)^{2m}}
\left[
\proba_{\sigma \in \Gamma_m} \left(
\sup_{f \in \bigF_{pq}} \vert \err_{\sigma(T)}^l [f] -
\err_{\sigma(T')}^l [f] \vert \geq \frac{\epsilon}{2}
\right)
\right], \label{eq:thdProba}
\end{align}
which concludes the second step.

\paragraph{\bf Reduction to a finite class.}
The idea is to reduce $\bigF_{pq}$ in
\eqref{eq:thdProba} to a finite class of functions.
For the sake of conciseness, we will not enter
into the details of the theory of \emph{covering numbers}.
Please refer to the corresponding literature for 
further details ({\em e.g.} \cite{DevroyeGL}).

In the following, $\bigN({\epsilon}/_8, \bigF_{pq}, 2m)$ will denote the
\emph{uniform ${\epsilon}/_8$convering number} of $\bigF_{pq}$
over a sample of size $2m$.

Let define $\bigG_{pq} \subset \bigF_{pq}$ such that $(l^{\bigG_{pq}})_{\vert
(T,T')}$ is an ${\epsilon/_8}$-cover of $(l^{\bigF_{pq}})_{\vert
(T,T')}$.
Thus, $\vert \bigG_{pq} \vert \leq \bigN( {\epsilon}/_8,
l^{\bigF_{pq}}, 2m ) < \infty$ Therefore, if $\exists f \in \bigF_{pq}$
such that $\vert \err_{\sigma(T)}^l [f] - \err_{\sigma(T')}^l [f]
\vert \geq \frac{\epsilon}{2}$ then, $\exists g \in \bigG_{pq}$
such that $\vert \err_{\sigma(T)}^l [g] - \err_{\sigma(T')}^l [g]
\vert \geq \frac{\epsilon}{4}$ and the following comes
naturally

\begin{align}
& \proba_{\sigma \in \Gamma_m} \left(
\sup_{f \in \bigF_{pq}} \vert \err_{\sigma(T)}^l [f] -
\err_{\sigma(T')}^l [f] \vert \geq \frac{\epsilon}{2}
\right) \notag \\
& ~ \leq \proba_{\sigma \in \Gamma_m} \left(
\max_{g \in \bigG_{pq}} \vert \err_{\sigma(T)}^l [g] -
\err_{\sigma(T')}^l [g] \vert \geq \frac{\epsilon}{4}
\right) \notag \\
& ~ \leq \bigN( {\epsilon}/_8, l^{\bigF_{pq}}, 2m ) \max_{g \in \bigG_{pq}}
\proba_{\sigma \in \Gamma_m} \left(
\vert \err_{\sigma(T)}^l [g] - \err_{\sigma(T')}^l [g] \vert 
\geq \frac{\epsilon}{8}
\right) \tag{union bound}
\end{align}

\paragraph{\bf Hoeffding's inequality.}
Finally, consider $\vert \err_{\sigma(T)}^l [g] - 
\err_{\sigma(T')}^l [g] \vert$ as the average of $m$
realizations of the same random variable, with expectation
equal to $0$. Then by Hoeffding's inequality we have
that\footnote{Note that in some references the right-hand side of
\eqref{eq:final} might viewed as a probability measure over $m$
independent Rademacher variables.}

\begin{align}
 & \proba_{\sigma \in \Gamma_m} \left(
\vert \err_{\sigma(T)}^l [g] - 
\err_{\sigma(T')}^l [g] \vert \geq \frac{\epsilon}{4}
\right) \leq 2 e^{{-m\epsilon^2}/_{128}} \label{eq:final}
\end{align}

Putting everything together yields the result w.r.t.
$\bigN ({\epsilon}/_8, l^{\bigF_{pq}}, 2m)$ for
$m\epsilon^2 \geq 32$. For $m\epsilon^2 < 32$
it holds trivially.

Recall	 that $l^{\bigF_{pq}}$ is Lipschitz in its first
argument with a Lipschitz constant $L = 1$ thus
$\bigN ({\epsilon}/_8, l^{\bigF_{pq}}, 2m) \leq 
\bigN ({\epsilon}/_8, \bigF_{pq}, 2m) = \bigO \left(
 \left( \frac{8}{\epsilon} \right)^{\Pdim(\bigF_{pq})} \right)$
\end{proof}

The last part of the proof comes from the observation that, for any fixed
$(p,q)$, we had never used any other specific information about $\bigF_{pq}$
other than the upper bound of $d+1$ over its pseudo dimension. In other words,
equation $\eqref{eq:result}$ holds for slightly modified definition of
$\bigF_{pq}$ as long as the pseudo dimension does not
exceed $d+1$.

Let us now consider :
$$ \H{\bigF_{pq}} \doteq \lbrace f: f(\V{x}) \doteq 
\indicator{t(\V{x}) = q} \indicator{\V{x} \in \setA{p}^{\alpha}}
\dotProd{\V{w}_p - \V{w}_q}{x}: \V{w}_p, \V{w}_q \in \mathcal{B}^d
\rbrace $$ 

Clearly for each function in $\H{\bigF_{pq}}$ there is
at most one corresponding affine function, thus $\H{\bigF_{pq}}$ and
$\bigF_{pq}$ share the same upper bound of $d+1$ on their pseudo-dimension.

Consequently, any covering number of $\bigF_{pq}$ is also a covering number of
$\H{\bigF_{pq}}$. More precisely, this proof holds true for any $\V{w}_p$ and
$\V{w}_q$, independently of $\setA{p}^{\alpha}$ which may itself be defined with
respect to $\V{w}_p$ and $\V{w}_q$.

%%% Keep this for now %%%
\iffalse
Clearly, for any fixed $(p,q)$ the same result holds as we
had never used any specific information about $\bigF_{pq}$ other than its pseudo
dimension. More precisely, this proof holds true for any $\V{w}_p$ and
$\V{w}_q$ independently of $\setA{p}^{\alpha}$ which may itself be defined with
respect to $\V{w}_p$ and $\V{w}_q$.

Indeed, for each function in $\H{\bigF_{pq}}$ there is
at most one corresponding affine function, therefore, $\H{\bigF_{pq}}$ and
$\bigF_{pq}$ share the same upper bound of $d+1$ on their pseudo-dimension.
\fi

It comes naturally that, fixing $S$ as the training set,
the following holds true:
$$\frac{1}{m} \sum_m  
\indicator{t(\V{x}) = q} \indicator{\V{x} \in \setA{p}^{\alpha}}
\V{x} =
\V{z}_{pq}.$$ Thus $$\left| \err_T^l [f] - \err_D^l [f] 
\right| = \left| \dotProd{\frac{\V{w}_p - \V{w}_q}
{\norm{\V{w}_p - \V{w}_q}}}{\V{z}_{pq}} - 
\dotProd{\frac{\V{w}_p - \V{w}_q}
{\norm{\V{w}_p - \V{w}_q}}}{\mupt{p}{q}} \right|.$$

We can generalize this result for any couple $(p,q)$
by a simple union bound, giving the desired 
inequality:
\begin{align*}\proba_{(\inputspace \times \outputspace) \sim \UKDistri} &\left(
\sup_{W \in \realset^{d \times Q}} \left| \dotProd{\frac{\V{w}_p - \V{w}_q}
{\norm{\V{w}_p - \V{w}_q}}}{\V{z}_{pq}} - 
\dotProd{\frac{\V{w}_p - \V{w}_q}
{\norm{\V{w}_p - \V{w}_q}}}{\mupt{p}{q}} \right|\geq \epsilon
\right) \\
 &\qquad\leq  \bigO \left(Q^2
 \left( \frac{8}{\epsilon} \right)^{(n+1)}
e^{ {m\epsilon^2}/_{128} } \right)
\end{align*}

Equivalently, we have that
$$ \left| \dotProd{\frac{\V{w}_p - \V{w}_q}
{\norm{\V{w}_p - \V{w}_q}}}{\V{z}_{pq}} - 
\dotProd{\frac{\V{w}_p - \V{w}_q}
{\norm{\V{w}_p - \V{w}_q}}}{\mupt{p}{q}} \right| \geq \epsilon $$
with probability $1 - \delta$ for 
$$ m \in \bigO\left( \frac{1}{\epsilon^2} \left[
\ln\left( \frac{1}{\delta} \right) + \ln(Q) + d\ln\left( \frac{1}{\epsilon} \right)
\right] \right).$$
\end{proof}

\end{document}